\documentclass[sigconf]{acmart}

\makeatletter

\renewcommand*{\@fnsymbol}[1]{%
  \ensuremath{\ifcase#1\or \dagger\or \ast\or \ddagger\or
     \mathsection\or \mathparagraph\or \|\or **\or \dagger\dagger
     \or \ddagger\ddagger \else\@ctrerr\fi}}
\makeatother

\AtBeginDocument{%
  }

\setcopyright{acmlicensed}
\copyrightyear{2026}
\acmYear{2026}
\setcopyright{cc}
\setcctype{by}
\acmConference[WWW '26]{Proceedings of the ACM Web Conference 2026}{April 13--17, 2026}{Dubai, United Arab Emirates}
\acmBooktitle{Proceedings of the ACM Web Conference 2026 (WWW '26), April 13--17, 2026, Dubai, United Arab Emirates}
\acmPrice{}
\acmDOI{10.1145/3774904.3792126}
\acmISBN{979-8-4007-2307-0/2026/04}

\usepackage{subfigure}
\usepackage{multirow}
\usepackage{multicol}
\usepackage{wrapfig}
\usepackage{amsthm}
\usepackage{balance}
\theoremstyle{plain}
\newtheorem{definition}{Definition}

\newtheorem{proposition}{Proposition}

\begin{document}

\title{LEDA: Latent Semantic Distribution Alignment for Multi-domain Graph Pre-training}

\author{Lianze Shan}
\authornote{Both authors contributed equally to this research.}
\email{shanlz2119@tju.edu.cn}
\affiliation{
  \institution{Tianjin University}
  \city{Tianjin}
  \country{China}
}

\author{Jitao Zhao}
\authornotemark[1]
\email{zjtao@tju.edu.cn}
\affiliation{
  \institution{Tianjin University}
  \city{Tianjin}
  \country{China}
}

\author{Dongxiao He}
\authornote{Corresponding author.}
\email{hedongxiao@tju.edu.cn}
\affiliation{%
  \institution{Tianjin University}
  \city{Tianjin}
  \country{China}
}

\author{Siqi Liu}
\email{siqiliu@tju.edu.cn}
\affiliation{%
  \institution{Tianjin University}
  \city{Tianjin}
  \country{China}
}

\author{Jiaxu Cui}
\email{cjx@jlu.edu.cn}
\affiliation{%
  \institution{Jilin University}
  \city{Changchun}
  \country{China}
}
\author{Weixiong Zhang}

\email{weixiong.zhang@polyu.edu.hk} 
\affiliation{%
  \institution{The Hong Kong Polytechnic University}
  \city{Kowloon}
  \country{Hong Kong}
}

\renewcommand{\shortauthors}{Lianze Shan et al.}

\begin{abstract}
Recent advances in generic large models, such as GPT and DeepSeek, have motivated the introduction of universality to graph pre-training, aiming to learn rich and generalizable knowledge across diverse domains using graph representations to improve performance in various downstream applications. However, most existing methods face challenges in learning effective knowledge from generic graphs, primarily due to simplistic data alignment and limited training guidance. The issue of simplistic data alignment arises from the use of a straightforward unification for highly diverse graph data, which fails to align semantics and misleads pre-training models. The problem with limited training guidance lies in the arbitrary application of in-domain pre-training paradigms to cross-domain scenarios. While it is effective in enhancing discriminative representation in one data space, it struggles to capture effective knowledge from many graphs. To address these challenges, we propose a novel Latent sEmantic Distribution Alignment (LEDA) model for universal graph pre-training. Specifically, we first introduce a dimension projection unit to adaptively align diverse domain features into a shared semantic space with minimal information loss. Furthermore, we design a variational semantic inference module to obtain the shared latent distribution. The distribution is then adopted to guide the domain projection, aligning it with shared semantics across domains and ensuring cross-domain semantic learning. LEDA exhibits strong performance across a broad range of graphs and downstream tasks. Remarkably, in few-shot cross-domain settings, it significantly outperforms in-domain baselines and advanced universal pre-training models.
\end{abstract}






\maketitle

\section{Introduction}
Graph pre-training has emerged as a powerful framework that aims to design proxy tasks and train encoders to learn knowledge with high generalizability \cite{graph-pretrain-survey}.
It has been widely applied in various  label-scarce scenarios, such as social networks \cite{social_network_2}, recommendation systems \cite{LightGCN}, and molecular analysis \cite{molecular_analysis}, and has played a pivotal role in graph learning research and development.
Recently, it has served as a core technique for developing graph foundation models \cite{MDGPT,samgpt,UniPrompt}, enabling them to acquire universal knowledge that enhances performance across diverse downstream tasks. 

Recent advances in graph pre-training have shifted from specific, application-oriented, self-supervised paradigms \cite{DGI,GRACE,BGRL} to generalizable and universal pre-training frameworks \cite{gcc,OFA,samgpt}.
Classical graph self-supervised pre-training frameworks train encoders to learn graph representations by designing contrastive or generative proxy tasks \cite{GRACE,BGRL,GVAE,GraphMAE}. 
While effective on applications within the same domains, these frameworks are difficult to generalize across diverse graphs due to the inherent high disparities of graph data \cite{FUG}.
With the advancement of universal large models \cite{bert,achiam2023gpt},
universal graph pre-training models have drawn much attention lately, aiming to learn more generalizable knowledge by training on a variety of graphs, enabling them to be readily applicable to many downstream tasks \cite{FUG,OFA,graphcontrol}.
To achieve this objective, early methods depended on structural encoding to ensure the generality of attribute-free graphs \cite{gcc}.
Therefore, these methods are restricted to structural generalizability.
More recent methods employ Language Models (LM) by converting node attributes to textual descriptions, enabling semantic encoding and cross-graph attribute alignment \cite{OFA}.
Another paradigm leverages Matrix Factorization (MF) methods, such as Principal Component Analysis (PCA) and Singular Value Decomposition (SVD), to project node features into a lower-dimensional space for alignment \cite{FUG,MDGPT}.
Based on the aligned data, most existing methods explore contrastive learning or link reconstruction objectives for model training \cite{gcc,FUG,MDGPT,OFA}.

\textbf{However, these methods face challenges in learning effective knowledge from numerous general graphs due to their simplistic data alignment strategies and limited knowledge acquisition capabilities.}
Firstly, simplistic alignment strategies lack the flexibility and expressiveness to handle highly diverse graphs, resulting in sub-optimal alignment when applied to generic graphs.
Specifically, structure-based methods fall short on attributed graphs as they ignore node feature semantics.
LM-based methods rely on textual attributes and are therefore inadequate for handling graphs with non-textual features.
MF-based methods align features by identifying principal variance directions but may retain noise or lose task-relevant semantics.
Moreover, the projected dimensions lack semantic clarity, which can potentially lead to confusion during pre-training.
Secondly, limited knowledge learning stems from the arbitrary application of in-domain pre-training paradigms to multi-domain scenarios.
Although the widely used contrastive learning and link reconstruction perform well in in-domain scenarios, they may become sub-optimal in multi-domain settings, as semantic ambiguity across samples makes it difficult to establish meaningful relations in the unified space. 
We further explore this issue in detail in Section \ref{sec:analysis_of_ugp}.
These limitations motivated us to consider \textit{how to design alignment strategies that preserve semantics across diverse graphs} and \textit{how to defince pre-training objectives that learn effective knowledge from multi-domain graphs.}

To address these challenges, we propose a novel \textbf{L}atent s\textbf{E}mantic \textbf{D}istribution \textbf{A}lignment (LEDA) model for multi-domain universal graph pre-training.
Specifically, we introduced a learnable domain projection unit that captures projection directions aligned with shared semantics, enabling the unification of multi-domain graphs.
Instead of relying on arbitrary variance-maximizing projections, we maximize the mutual information between original features and their projected representations to ensure each dimension retains essential information from the original feature.
Building upon the aligned feature, we further incorporate a latent distribution alignment module, which explicitly aligns the posterior distributions of individual domains with a shared prior.
This process also provides semantic guidance for the domain projection unit, promoting alignment along consistent directions across different domains.
We evaluate LEDA under multiple settings, including cross-domain node classification, few-shot scenarios, and cross-domain zero-shot graph classification. Experimental results across these settings demonstrate the effectiveness of LEDA.
Our contributions can be summarised as follows:

\begin{itemize}
    \item We revealed that existing methods are incapable of learning generalizable knowledge from multi-domain graphs due to their simplistic data alignment strategies and limited knowledge learning. We adopted a variational semantic modeling perspective to overcome these challenges.

    \item We proposed LEDA, a novel universal graph pre-training approach, which integrates a domain projection unit for adaptive feature projection and a latent distribution alignment module to jointly optimize latent inference and generalizable semantic modeling.
    
    \item We conducted extensive experiments to assess the effectiveness of our model, including cross-domain node classification, cross-domain graph classification, and few-shot scenarios using seven widely-used node classification datasets and four representative graph classification benchmarks.
\end{itemize}

\section{Related Work}
\label{appendix:related work}
\textbf{Graph Representation Learning.} Graphs are ubiquitous in real-world scenarios, such as  social networks \cite{social_network_2}, bioinformatics \cite{bioinformatics} and recommendation systems \cite{LightGCN}.
This leads to an extensive focus on Graph Representation Learning (GRL), which aims to encode graph data into low-dimensional continuous embedding space for every step in downstream tasks \cite{grl_survey}. 
Early works rely on shallow embedding methods, including matrix factorization \cite{Matrix_Factorization} and random walk-based \cite{randomwalk} methods, which struggle to jointly encode node features and topology. 
In recent years, Graph Neural Networks (GNNs) have achieved significant success, as demonstrated by models such as GCN \cite{GCN}, GAT \cite{GAT} and GraphSAGE \cite{GraphSAGE}, which effectively combine node features and topology simultaneously through message passing and aggregation mechanisms. 
These methods have achieved great success in a wide range of downstream tasks, including node classification, link prediction, and graph classification. 
However, training effective GNNs requires large amounts of manually labeled data, which is extensive and time-consuming to obtain. 

\textbf{Graph Pre-training.} To overcome the limitations of GNNs in label-scarce scenarios, graph pre-training has emerged as an effective paradigm for learning general knowledge and latent graph patterns \cite{graph-pretrain-survey}. 
As an early paradigm of graph pre-training, Graph Self-Supervised Learning (GSSL) has been extensively studied \cite{SGRL,E2Neg,Str-GCL}. 
GSSL mines data itself to generate self-supervised samples and designs proxy tasks for training GNN encoders in a self-supervised way. 
Recent advances in GSSL have predominantly evolved along two directions: (i) contrastive-based methods \cite{GRACE,DGI,BGRL}, which aim to maximize mutual information between different views of original graph, and (ii) generative-based methods \cite{GVAE,MaskGAE,GraphMAE}, which focus on reconstructing masked graph signals. 
While these methods have achieved remarkable success in label-scarce scenarios, most of them remain confined to in-domain distribution, making it difficult to extract transferable knowledge across diverse graphs. 

Inspired by the success of unified pre-training large models \cite{bert}, recent efforts have focused on developing unified graph pre-training models. 
Given the reliance of GNNs on fixed input dimensions and the semantic inconsistency of node features across graphs, some methods \cite{gcc,unilink} discard raw node features and instead encode topology-based positional information derived purely from graph structure.
While these methods enable better generalization across attribute-free graphs, their inability to incorporate attribute information results in suboptimal performance on attributed graphs, which are common in real-world scenarios.
To unify attribute information across graphs, methods such as OFA \cite{OFA} and ZEROG \cite{zeroG} convert graphs into text and utilize Language Models (LMs) to re-encode node attributes. 
Building on the unified input, OFA \cite{OFA} introduces virtual nodes to integrate essential information derived from the original graph and LMs, which is then propagated via GNNs. 
ZEROG \cite{zeroG} proposes a prompt-based subgraph sampling mechanism that captures semantic relevance through selected prompt nodes, while preserving structural characteristics via local neighborhood aggregation.
However, these methods are limited by the assumption that node attributes can be effectively textualized, and the use of LMs incurs substantial computational and resource costs \cite{FUG}. 
More recent studies have explored input unification across general graphs by applying dimension reduction techniques such as Principal Component Analysis (PCA)  and Singular Value Decomposition (SVD) to align node feature dimensions.
FUG \cite{FUG} revisits the theoretical connection between PCA and contrastive learning, introducing a dimensional encoder to achieve lossless feature alignment. While effective in unifying input dimensions, it heavily relies on data sampling and struggles to capture shared semantics across domains.
MDGPT \cite{MDGPT} aligns diverse feature dimensions by SVD and introduces the concept of domain tokens to unify semantics from multi-domain graphs.
However, the number of domain tokens increases with the number of input domains, which limits the model's flexibility and scalability.

Despite the progress made in recent efforts on universal graph pre-training, existing approaches still face intrinsic challenges in multi-domain settings. These limitations motivate us to reconsider how to design a universal graph pre-training framework that can effectively learn transferable knowledge from multi-domain graphs in a scalable and generalizable manner.

\section{Notations and Preliminary}
\label{Sec:Notations and Preliminary}
We now present the preliminary concepts and notations used in the paper. Sets are denoted by calligraphic letters (e.g., $\mathcal{G}$), matrices are represented in bold capital letters (e.g., \(\mathbf{X}\)), vectors are expressed in bold lowercase letters (e.g., $\mathbf{x}$, which denotes the vector of matrix $\mathbf{X}$), and scalars are denoted by lowercase letters (e.g., \(x\)). 

\textbf{Graph data.}
Given a graph $\mathcal{G} = (\mathcal{V}, \mathcal{E})$, we define $\mathcal{V}$ as the node set, where \(\mathcal{E} \subseteq \mathcal{V} \times \mathcal{V}\) denotes the edge set. Each node \(v_i\) is associated with a feature vector \(\mathbf{x}_i\), and the collective features of all nodes in graph \(\mathcal{G}\) form the feature matrix matrix \(\mathbf{X} \in \mathbb{R}^{n \times d}\), where $n$ is the number of nodes and \(d\) is the feature dimension. The topological structure of \(\mathcal{G}\) is specified  by its adjacency matrix \(\mathbf{A} \in \{0,1\}^{n \times n}\), where the entry \(a_{ij}\) is defined as $a_{ij} = 1$ if $e_{ij} \in \mathcal{E}$ or $a_{ij} = 0$ otherwise.

\textbf{Graph Neural Networks (GNNs).} Graph pre-training involves training a universal encoder that can extract generalizable knowledge from diverse graphs that can be rapidly adapted and transferred to a variety of downstream tasks.
In most graph pre-training methods, GNNs \cite{GNNsurvey1, GNNsurvey2} are used as the backbone encoder, which encode node representations through the message passing and aggregation framework:
\begin{equation}
\label{eq:gnn_encoder}
\mathbf{h}_i^{(l+1)} = \text{UPDATE}^{(l)}\left( \mathbf{h}_i^{(l)}, \text{AGGREGATE}^{(l)}\left( \{ \mathbf{h}_j^{(l)} : j \in \mathcal{N}(i) \} \right) \right),
\end{equation}
where $\mathbf{h}_i^{(l)}$ denotes the $l$-layer representation of node $v_i$, $\mathcal{N}(i)$ denotes the set of neighbors of $v_i$, $\text{AGGREGATE}^{(l)}$ and $\text{UPDATE}^{(l)}$ are layer-specific functions, and the primary differences among various GNNs lie in how these two functions are designed.
In our work, we consider the Graph Convolutional Network (GCN) \cite{GCN} as the backbone encoder for graph pre-training, which defines the aggregation as a normalized sum over neighbors:
\begin{equation}
\label{eq:gcn_encoder}
\mathbf{H}^{(l+1)} = \sigma\left( \tilde{\mathbf{D}}^{-\frac{1}{2}} \tilde{\mathbf{A}} \tilde{\mathbf{D}}^{-\frac{1}{2}} \mathbf{H}^{(l)} \mathbf{W}^{(l)} \right),
\end{equation}
where $\tilde{\mathbf{A}} = \mathbf{A} + \mathbf{I}$ is the adjacency matrix with self-loops, and $\tilde{\mathbf{D}}$ is the corresponding diagonal degree matrix, $\mathbf{W}^l$ is a learnable weight matrix and $\sigma$ is a non-linear activation function.

\textbf{Multi-domain Graph Pre-training.} 
Training the universal GNN encoder mentioned above requires pre-training on multi-domain graphs to learn generalizable knowledge.
And we define the collection of multi-domain graphs as $\mathcal{D}_{\mathcal{G}} = \{\mathcal{G}^{(1)}, \mathcal{G}^{(2)}, \dots, \mathcal{G}^{(M)}\}$, where each $\mathcal{G}^{(k)} = (\mathcal{V}^{(k)}, \mathcal{E}^{(k)})$ denotes the $k$-th graph instance in the set, and $M$ is the number of graphs. These graphs often exhibit distinct characteristics, particularly in node feature dimensions and semantics.
This poses challenges for training a universal GNN encoder due to the fixed input requirement of $\mathbf{W}^l$. To address this, some methods apply data pre-processing techniques such as singular value decomposition to align features:
\begin{equation}
\label{eq:SVD}
\mathbf{X}^i \approx \mathbf{U}^i_k \mathbf{\Sigma}^i_k \mathbf{V}_k^{i^\top}, \mathbf{X}^i_{\text{align}}=\mathbf{X}^i\mathbf{V}^i_k,
\end{equation}
where $\mathbf{U}^i_k \in \mathbb{R}^{n \times k}$, $\mathbf{\Sigma}^i_k \in \mathbb{R}^{k \times k}$, and $\mathbf{V}^{i}_k \in \mathbb{R}^{d \times k}$ are the top-$k$ singular vectors and values. The reduced representation $\mathbf{X}^i_{\text{align}}$ is then used as the aligned node features.

\section{Methodology}
\label{sec:methodology}

In this section, we introduce our \textbf{L}atent s\textbf{E}mantic \textbf{D}istribution \textbf{A}lignment (LEDA) approach for multi-domain graph pre-training, as illustrated in Figure \ref{fig:overview}. 
The goal of LEDA is to learn generalizable distribution from multi-domain graphs, enabling the encoder directly applied to various downstream tasks.
We first design a \textbf{D}omain \textbf{P}rojection \textbf{U}nit (DPU) to provide a semantically aligned basis that maps multi-domain features into a unified embedding space while preserving their information through mutual information maximization(in Section \ref{sec:DPU}). 
We then conduct a detailed analysis of the limitation of existing graph pre-training methods in multi-domain scenarios (in Section \ref{sec:analysis_of_ugp}).
Built upon this, we introduce a \textbf{L}atent \textbf{D}istribution \textbf{A}lignment module (LDA), which learns a shared latent distribution across domains by aligning their posterior distribution to a common prior (in Section \ref{sec:LDA}). 
In addition, a detailed discussion of the assumptions used in this section is provided in Appendix \ref{app_supplement_dis}.

\begin{figure*}[t]
\centering
\includegraphics[width=1\textwidth]{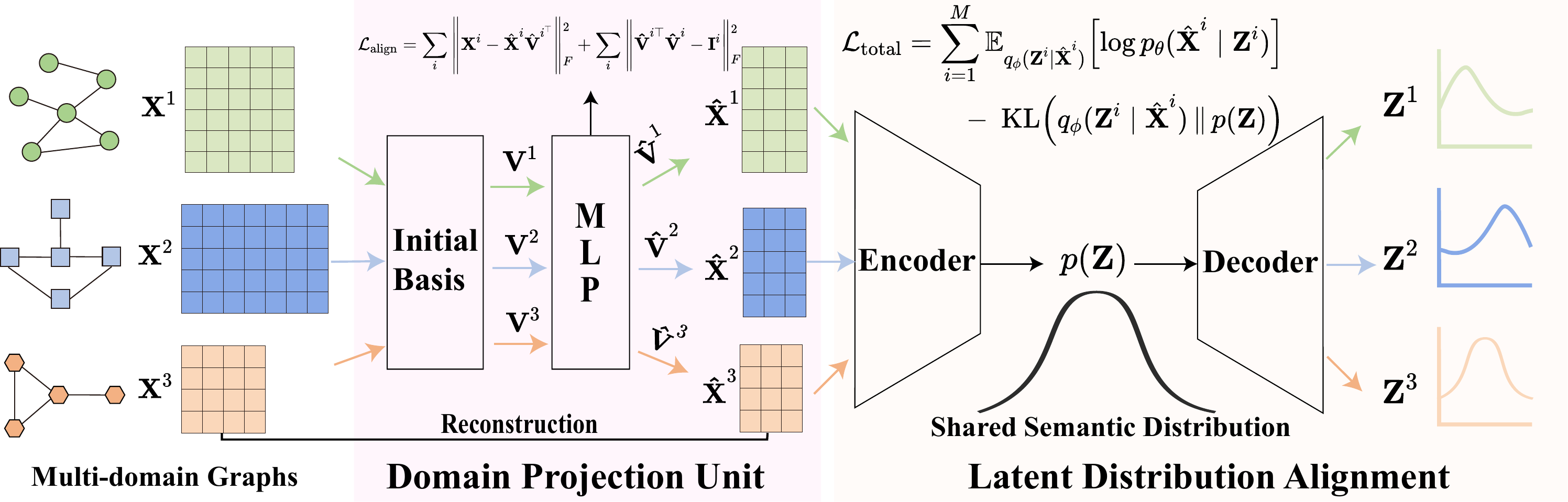}
\caption{Overview of the proposed LEDA. Given the train dataset $ \mathcal{T}_\mathcal{G}=\{\mathcal{G}^1, \mathcal{G}^2, \dots, \mathcal{G}^t\}$, we first get their initial projection basis set $\mathcal{P}_\mathcal{G}=\{\textbf{V}^1,\textbf{V}^2, \dots, \textbf{V}^t\}$ by SVD. Subsequently, this set of projection basis is processed by a learnable multi-layer perceptron, which is optimized by $\mathcal{L}_{\text{align}} = \mathcal{L}_{\text{recon}} + \lambda \cdot \mathcal{L}_{\text{ortho}}$. 
Furthermore, we encode the unified feature processed by DPU through a single-layer GCN and align the posterior distribution with a shared latent distribution. Finally, we jointly optimize the parameters of DPU and LDA using the loss function $\mathcal{L}_{\text{total}}$.
}
\label{fig:overview}
\end{figure*}

\subsection{Domain Projection Unit}
\label{sec:DPU}

A key challenge in multi-domain graph pre-training is the semantic heterogeneity of node features: domains may contain different feature dimensions and semantics, making direct comparison or joint learning ineffective. 
While dimensionality reduction can unify feature dimensions, it ignores two critical requirements in multi-domain settings: 
(i) enabling cross-domain semantic alignment while retaining essential semantics, and (ii) avoiding domain semantic conflicts. More discussion about domain semantic conflicts is in Appendix \ref{app_supplement_dis}.
To address these challenges, we propose a novel \textbf{D}omain \textbf{P}rojection \textbf{U}nit (DPU), a learnable projection module that maps multi-domain features into a shared semantic subspace while satisfying both requirements through $\mathcal{L}_\text{align}$.

Specifically, for each graph $\mathcal{G}^i = (\mathbf{X}^i, \mathbf{A}^i)$, we first compute an initial projection basis $\mathbf{V}^i \in \mathbb{R}^{d_i \times k}$ via SVD to capture the dominant variance in $\mathbf{X}^i$. 
This provides a stable, low-rank initialization that avoids the instability of random projection in high dimensions.
We then refine $\mathbf{V}^i$ using a domain-shared learnable transformation function $\mathrm{Trans}(\cdot)$, implemented as a two-layer MLP:
\begin{equation}
\label{eq:dpu_encode}
    \mathbf{\hat{V}}^i = \mathrm{Trans}(\mathbf{V}^i) =  \mathrm{ReLU}(\mathbf{V}^i \mathbf{W}_1  + \mathbf{b}_1) \cdot \mathbf{W}_2  + \mathbf{b}_2 ,
\end{equation}
where the same parameters $\{\mathbf{W}_1, \mathbf{W}_2, \mathbf{b}_1, \mathbf{b}_2\}$ are applied to all domains. 
This shared-parameter forces all domains to adapt their basis into a unified semantic space, laying the foundation for cross-domain compatibility.
The aligned node representations are obtained as $\mathbf{\hat{X}}^i = \mathbf{X}^i \mathbf{\hat{V}}^i$, where $\mathbf{\hat{X}}^i \in \mathbb{R}^{n_i \times m}$.
To ensure that this projection retains essential semantics, we introduce a reconstruction loss that encourages $\mathbf{\hat{X}}^i$ to be sufficient for recovering $\mathbf{X}^i$:
\begin{equation}
\label{eq:recon}
    \mathcal{L}_{\text{recon}} = \sum_i \left\| \mathbf{X}^i - \hat{\mathbf{X}}^i \hat{\mathbf{V}}^{i^\top} \right\|_F^2 = \sum_i \left\| \mathbf{X}^i - \mathbf{X}^i \mathbf{\hat{V}}^i\hat{\mathbf{V}}^{i^\top} \right\|_F^2,
\end{equation}
where $\left\| \cdot \right\|_F $ denotes the Frobenius norm. 

However, reconstruction and shared parameters alone are insufficient to resolve the semantic conflicts in the unified embedding space.
To address this, we enhance the discriminability of the aligned representations by maximizing the mutual information between the original features $\mathbf{X}^i$ of each domain and their aligned representations $\hat{\mathbf{X}}^i$ as follows:
\begin{equation}
\label{eq:information_bottle}
    I(\mathbf{X}^i; \mathbf{\hat{X}}^i) = H(\mathbf{\hat{X}}^i) - H(\mathbf{\hat{X}}^i \mid \mathbf{X}^i),
\end{equation}
where $H(\mathbf{\hat{X}}^i)$ denotes the entropy of the projected features, and $H(\mathbf{\hat{X}}^i \mid \mathbf{X}^i)$ represents the conditional entropy given the original features. Given that the projected representation $\hat{\mathbf{X}}^i$ is deterministically computed from the original features $\mathbf{X}^i$ (i.e., $\hat{\mathbf{X}}^i = \mathbf{X}^i \cdot f(\mathbf{V}^i)$), where $f(\cdot)$ is a deterministic function and $\mathbf{V}^i$ is derived from $\mathbf{X}^i$, the conditional entropy $H(\hat{\mathbf{X}}^i \mid \mathbf{X}^i)$ becomes zero. 
Therefore, to maximize the mutual information, we maximize the entropy of the projected representation $\hat{\mathbf{X}}^i = \mathbf{X}^i \cdot f(\mathbf{V}^i)$.
Since $\mathbf{X}^i$ is deterministic for each domain $\mathcal{G}^i$, this is achieved by encouraging higher entropy in $\hat{\mathbf{V}}^i$. Assuming $\mathbf{\hat{V}}^i$ follows an approximately multivariate Gaussian distribution, the entropy admits the following expression:
\begin{equation}
\label{eq:differential entropy}
    h(\mathbf{\hat{V}}^i) = \frac{1}{2} \log \left( (2\pi e)^m \cdot \det(\boldsymbol{\Sigma}_{\hat{V}^i}) \right),
\end{equation}
where $\boldsymbol{\Sigma}_{\hat{V}^i}$ is the covariance matrix and $\det(\cdot)$ denotes the determinant of the covariance matrix.
According to the \textit{Hadamard inequality}, the determinant of $\boldsymbol{\Sigma}_{\hat{V}^i}$ is maximized when the columns of $\mathbf{\hat{V}}^i$ are orthogonal. 
Therefore, enforcing orthogonality among projection directions increases $h(\mathbf{\hat{V}}^i)$, thus indirectly improving the mutual information $I(\mathbf{X}^i; \mathbf{\hat{X}}^i)$ and enhancing the semantic quality of the representations. 
We encode this constraint using the following regularization term:
\begin{equation}
\label{eq:loss_ortho}
    \mathcal{L}_{\text{ortho}} = \sum_i \left\| \mathbf{\hat{V}}^{i\top} \mathbf{\hat{V}}^i - \mathbf{I}^i \right\|_F^2,
\end{equation}
and the final alignment loss is defined as:
\begin{equation}
\label{eq:align}
    \mathcal{L}_{\text{align}} = \mathcal{L}_{\text{recon}} + \lambda \cdot \mathcal{L}_{\text{ortho}},
\end{equation}
where $\lambda$ is a hyperparameter that balances information preservation and structural regularization. 
This joint objective ensures that DPU learns a unified embedding space where features can support cross-domain semantic consistency. 

\subsection{Analysis of graph pre-training}
\label{sec:analysis_of_ugp}
Existing graph pre-training methods, such as contrastive learning and link prediction \cite{DGI,GRACE,BGRL,GraphMAE}, are primarily designed for single-domain settings. 
When applied to multi-domain graphs, they often fail to capture shared semantics across domains due to fundamental misalignments in their optimization objectives.

\textbf{Limitations of Contrastive Learning.} In contrastive learning frameworks (e.g., InfoNCE \cite{InfoNCE}), the goal is to maximize the similarity between positive node pairs (same domain) while minimizing it for negative pairs (different domains). This is formalized as:
\begin{equation}
\label{eq:contrastive_learning}
\begin{aligned}
 & \mathcal{L}_{\text{InfoNCE}} = - \frac{1}{N} \sum_{i=1}^{N} \log  ( \\ &\frac{\sum_{(v_i,v_j)\in \text{Pos}} \exp(\text{sim}(\mathbf{z}_i, \mathbf{z}_j^+)/\tau)}{\sum_{(v_i,v_j)\in \text{Pos}} \exp(\text{sim}(\mathbf{z}_i, \mathbf{z}_j^+)/\tau) + \sum_{(v_i,v_j)\in \text{Neg}} \exp(\text{sim}(\mathbf{z}_i, \mathbf{z}_j^-)/\tau)} ),
\end{aligned}
\end{equation}
where $N$ is the total number of nodes across all domains, $\mathbf{z} = \text{GNN}(\mathbf{A}, \mathbf{X})$, $(v_i,v_j) \in \text{Pos}$ when $v_i$ and $v_j$ come from the same domain, $(v_i,v_j) \in \text{Neg}$ when $v_i$ and $v_j$ come from different domains, $\text{sim}(\cdot)$ is the cosine similarity function, and $\tau$ is the temperature coefficient.
While effective within a single domain, this objective becomes problematic in multi-domain settings. Inspired by \cite{MINE}, we analyze the mutual information across domains to quantify cross-domain alignment.
\begin{definition}
(Mutual Information across Domains)
\label{defintion}
Let $\mathcal{D}_{\mathcal{G}^i}$ and $\mathcal{D}_{\mathcal{G}^j}$ denote the data distributions from two different domains $\mathcal{G}^i$ and $\mathcal{G}^j$. The mutual information across domains is defined as:
\begin{equation}
I(\mathcal{D}_{\mathcal{G}^i}; \mathcal{D}_{\mathcal{G}^j}) = \mathbb{E}_{x \sim \mathcal{D}_{\mathcal{G}^i} , x' \sim \mathcal{D}_{\mathcal{G}^j}} \left[ \log \frac{p(x, x')}{p(x) p(x')} \right],
\end{equation}
where $p(x, x')$ is the joint distribution over samples from different domains, and $p(x),p(x')$ are marginal distributions.
\end{definition}

In contrastive learning, the joint distribution $p(x, x')$ is typically modeled via a similarity-based scoring function $s_{ij}=f(x,x')$, normalized by a constant $Z:p(x, x')=\frac{e^{s_{ij}}}{Z}$. And the marginal distribution is modeld by $p(x)p(x') = C \cdot e^{\xi(x, x')}$, where $C$ is a constant and $e^{\xi(x, x')}$ is a bias term reflecting the difference between the assumption $ p(x) p(x') = C$ and the true distribution.
Based on this assumption, we give the following proposition:
\begin{proposition}
\label{prop:mi_contrastive}
Let the joint distribution over samples from two domains $\mathcal{D}_{\mathcal{G}^i}$ and $\mathcal{D}_{\mathcal{G}^j}$ be modeled as $p(x, x') = e^{s_{ij}} / Z$, where $s_{ij} = f(x, x')$ is a similarity score and $Z = \sum_{x, x'} e^{s_{ij}}$ (or $\int e^{s_{ij}} dx dx'$ in continuous case) is the normalization constant. Further assume the product of marginals satisfies $p(x)p(x') = e^{\xi(x, x')}$. Then the mutual information between the two domains satisfies:
\begin{equation}
    I(\mathcal{D}_{\mathcal{G}^i}; \mathcal{D}_{\mathcal{G}^j}) = \mathbb{E}[s_{ij}] - \log Z - \Delta,
\end{equation}
where $\Delta = \mathbb{E}[\xi(x, x')]$.
\end{proposition}

The proof is in Appendix. Proposition~\ref{prop:mi_contrastive} reveals a limitation of contrastive learning in multi-domain settings: the cross-domain mutual information is directly governed by the expected similarity score $\mathbb{E}[s_{ij}]$. 
In practice, contrastive objectives explicitly minimize $s_{ij}$ for cross-domain (negative) pairs to enforce domain separation.
As a result, $\mathbb{E}[s_{ij}]$ decreases, leading to a reduction in $I(\mathcal{D}_{\mathcal{G}^i}; \mathcal{D}_{\mathcal{G}^j})$.
While this enhances intra-domain discrimination, it actively suppresses the learning of shared semantics across domains.

\textbf{Limitations of Link Prediction.} Link prediction-based pre-training optimizes a local reconstruction objective: it encourages the model to predict observed edges by maximizing the similarity of connected node pairs. While effective for capturing domain-specific structural patterns, this approach is inherently limited in multi-domain settings for two key reasons. First, it overfits to local topology, ignoring higher-order semantic relationships that are often shared across domains (e.g., functional roles or community memberships). 
Second, and more critically, it lacks any explicit mechanism for cross-domain alignment: each domain is trained in isolation, so semantically equivalent nodes (e.g., “influential users” in social networks and “highly cited papers” in citation graphs) may be mapped to distant regions in the embedding space.
Consequently, link prediction learns representations that are highly specialized to single domain structures but fail to generalize to unseen domains or tasks requiring cross-domain knowledge.

\subsection{Latent Distribution Alignment}
\label{sec:LDA}
\textbf{Motivation of latent distribution alignment.} Existing graph pre-training methods often fail to explicitly capture shared semantics across domains, as they either focus on structural similarity or rely on domain-specific objectives. To address this, we propose to learn a shared latent semantic distribution $p(\mathbf{Z})$ from multi-domain graphs, under the assumption that all observed graphs $\{\mathcal{G}^1, \dots, \mathcal{G}^g\}$ are generated from this common prior \cite{MDP-GNN}, with domain-specific variations captured by conditional decoders $p_\theta(\hat{\mathbf{X}}^i \mid \mathbf{Z}^i)$.
Based on this assumption, our Latent Distribution Alignment (LDA) module aligns the posterior distribution $q_\phi(\mathbf{Z}^i \mid \hat{\mathbf{X}}^i)$, which is encoded by a shared GCN encoder, to the shared prior $p(\mathbf{Z})$ by minimizing the KL divergence.
Because the encoder is shared and the inputs have already been aligned by the DPU, this process enforces that semantically similar samples from distinct domains share a common latent representation. 
As a result, LDA promotes domain-invariant semantics.
In contrast, contrastive learning tends to arbitrarily push cross-domain pairs apart by treating them as negatives, and link prediction focuses mainly on local structural patterns rather than global semantic alignment.

\textbf{Instantiation of LDA.} Guided by the above analysis, we instantiated a latent distribution alignment module to learn generalizable knowledge from multi-domain graphs. 
Specifically, we encode each graph $\mathcal{G}^i = (\hat{\mathbf{X}}^i, \mathbf{A}^i)$ using a single-layer GCN encoder with parameters shared across domains:
\begin{equation}
    \mathbf{Z}^i = \mathrm{GCN}_{\text{base}}(\hat{\mathbf{X}}^i, \mathbf{A}^i) = \mathrm{ReLU}(\tilde{\mathbf{D}}^{-1/2} \tilde{\mathbf{A}} \tilde{\mathbf{D}}^{-1/2} \hat{\mathbf{X}}^i \mathbf{W}),
\end{equation}
where $\tilde{\mathbf{A}}^i = \mathbf{A}^i + \mathbf{I}^i$ denotes the adjacency matrix with self-loops, and $\tilde{\mathbf{D}}$ is the corresponding degree matrix. 
\begin{table*}[h]
\centering
\caption{
Accuracy (\%) of cross-domain node classification with standard deviations. Each column represents a test domain, while others are train domains. \textbf{CD} means whether the model is trained on multi-domain datasets and tested on a cross-domain dataset. Aside from GCN, the best results are bolded and the second-best are underlined. Methods with * are reported from \cite{FUG}.}
\resizebox{1.0\textwidth}{!}{
\begin{tabular}{lc|cccccccc}
\toprule
\textbf{Method}   &\textbf{CD} & \textbf{Cora} & \textbf{CiteSeer} & \textbf{PubMed}  & \textbf{Photo} & \textbf{Computers} & \textbf{CS} & \textbf{Physics} \\
\midrule
Raw features* 	& $\times$ & 57.90 $\pm$ 3.90 & 57.60 $\pm$ 2.85 & 79.55 $\pm$ 0.98 & 80.99 $\pm$ 1.65 & 75.59 $\pm$ 1.69 & 89.92 $\pm$ 0.95  & 93.18 $\pm$ 0.45 \\
SVD* & $\times$ & 56.36 $\pm$ 4.14 & 48.29 $\pm$ 3.18 & 82.80 $\pm$ 0.91 & 74.92 $\pm$ 1.82 & 71.26 $\pm$ 1.34  & 87.12 $\pm$ 0.73 & 92.48 $\pm$ 0.47 \\
DeepWalk* & $\times$ & 75.70 $\pm$ 0.00 & 50.50 $\pm$ 0.00 & 80.50 $\pm$ 0.00  & 89.44 $\pm$ 0.00 & 85.68 $\pm$ 0.00 & 84.61 $\pm$ 0.00 & 91.77 $\pm$ 0.00 \\
\midrule
GRACE* & $\times$ & 83.20 $\pm$ 1.87 & 70.99 $\pm$ 2.29 & 85.46 $\pm$ 0.54 & 91.93 $\pm$ 0.83 & 85.36 $\pm$ 0.82  & 91.84 $\pm$ 0.37 & OOM \\
GCA* & $\times$ & 82.83 $\pm$ 2.29 & 72.06 $\pm$ 1.91 & \underline{85.69 $\pm$ 0.68}  & 92.63 $\pm$ 1.12 & 87.78 $\pm$ 0.78 & 92.69 $\pm$ 0.49 & OOM \\
DGI* & $\times$ & 83.24 $\pm$ 2.12 & 71.23 $\pm$ 2.37 & 84.62 $\pm$ 0.83  & 92.32 $\pm$ 0.49 & 86.12 $\pm$ 0.73 & 92.47 $\pm$ 0.60 & 94.47 $\pm$ 0.50 \\
BGRL* & $\times$ & 81.57 $\pm$ 2.07 & 70.10 $\pm$ 2.04 & 83.67 $\pm$ 0.84  & 92.34 $\pm$ 0.73 & 86.51 $\pm$ 1.53 & 92.12 $\pm$ 0.63 & 95.42 $\pm$ 0.41 \\
GraphMAE & $\times$ &83.86 $\pm$ 0.62 & 72.26 $\pm$ 0.44 & 82.99 $\pm$ 0.21  & 92.21 $\pm$ 0.24 &87.48 $\pm$ 0.28 &92.28 $\pm$ 0.33 & 95.33 $\pm$ 0.21 \\
FUG* & $\times$ & \underline{84.45 $\pm$ 2.45} & \underline{72.43 $\pm$ 2.92} & 85.47 $\pm$ 1.13  & \underline{93.07 $\pm$ 0.82} &88.42 $\pm$ 0.98 & \underline{92.89 $\pm$ 0.45} & \underline{95.45 $\pm$ 0.27}\\ 
\midrule
OFA* & \checkmark & 75.90 $\pm$ 1.26 &- & 78.25 $\pm$ 0.71 &- &- &- &- \\
GraphControl* &\checkmark &- &- &-  &89.66 $\pm$ 0.56 &- &- &94.31 $\pm$ 0.12 \\
FUG & \checkmark & 83.58 $\pm$ 1.21 & 69.65 $\pm$ 1.74 &84.80 $\pm$ 0.30  &92.51 $\pm$ 0.20 &\textbf{89.29 $\pm$ 0.05} &92.45 $\pm$ 0.14 & 95.37 $\pm$ 0.33\\ 
\textbf{LEDA (Ours)}& \checkmark & \textbf{84.71 $\pm$ 0.73} & \textbf{73.40 $\pm$ 0.15} & \textbf{87.10 $\pm$ 0.27}  & \textbf{93.24 $\pm$ 0.98} & \underline{88.77 $\pm$ 0.56} & \textbf{93.55 $\pm$ 0.42} & \textbf{95.70 $\pm$ 0.29} \\
\midrule
Supervised GCN  & $\times$  & 82.80 $\pm$ 0.00 & 72.00 $\pm$ 0.00 & 84.80 $\pm$ 0.00  & 92.42 $\pm$ 0.00 & 86.51 $\pm$ 1.00 & 93.03 $\pm$ 0.00 & 95.65 $\pm$ 0.00 \\
\bottomrule
\end{tabular}
}

\label{tab:multi_domain_10test}
\end{table*}
We treat $\mathbf{Z}^i$ as the semantic basis for approximating the domain-specific variational posterior $q_\phi(\mathbf{Z}^i \mid \hat{\mathbf{X}}^i)$. Following the variational autoencoder framework \cite{GVAE}, we use two parallel GCNs (also shared across domains) to estimate the mean and log-variance:
\begin{equation}
    \boldsymbol{\mu}^i = \mathrm{GCN}_\mu(\mathbf{Z}^i, \mathbf{A}^i), \quad \log \boldsymbol{\sigma}^i = \mathrm{GCN}_\sigma(\mathbf{Z}^i, \mathbf{A}^i),
\end{equation}
where $\boldsymbol{\mu}^i$, $\boldsymbol{\sigma}^i$ represent the mean and variance of the domain-specific posterior distribution. 
These define a Gaussian posterior for each domain $\mathcal{G}^i$, formulated as $q_\phi(\mathbf{Z}^i \mid \hat{\mathbf{X}}^i) = \mathcal{N}(\mu^i, (\sigma^i)^2)$.
To enable gradient backpropagation, we follow the previous work \cite{GVAE} and employ the reparameterization trick to sample latent codes from the variational distribution:
\begin{equation}
    \mathbf{Z}^i = \boldsymbol{\mu}^i +\boldsymbol{\sigma} ^i \odot \boldsymbol{\epsilon}, \quad \boldsymbol{\epsilon} \sim \mathcal{N}(0, \mathbf{I}),
\end{equation}
where $\boldsymbol{\epsilon} \sim \mathcal{N}(0, \mathbf{I})$ is standard Gaussian noise, and $\odot$ denotes element-wise multiplication.
The latent codes are then decoded by a shared GCN decoder to reconstruct the DPU-aligned features:
\begin{equation}
    \hat{\mathbf{X}}_{\text{rec}}^i = \mathrm{Decoder}_\theta(\mathbf{Z}^i, \mathbf{A}^i),
\end{equation}
where $\mathrm{Decoder}_\theta(\cdot)$ is a shared GCN-based decoder parameterized by $\theta$.
To learn a shared semantic distribution across domains, we optimize a multi-domain variational objective that jointly trains all graphs under a common generative prior. The objective is given by:
\begin{equation}
\label{eq:ELBO}
\mathcal{L}_{\text{total}} =
\sum_{i=1}^{M} \underbrace{ \mathbb{E}_{q_\phi(\mathbf{Z}^i \mid \hat{\mathbf{X}}^i)} \left[ \log p_\theta(\hat{\mathbf{X}}^i \mid \mathbf{Z}^i) \right] }_{\text{Domain-specific reconstruction}} 
\ -\ 
\underbrace{ \mathrm{KL} \left( q_\phi(\mathbf{Z}^i \mid \hat{\mathbf{X}}^i) \,\|\, p(\mathbf{Z}) \right) }_{\text{Cross-domain alignment}},
\end{equation}
where $p(\mathbf{Z})$ is a shared prior that acts as a universal reference for all domains. 
The gradient computed from $\mathcal{L}_\text{total}$ is simultaneously propagated to both the GCN encoders and the transformation function in DPU.
Although the mathematical form resembles the standard ELBO, its role in LEDA is fundamentally different. 
In VGAE \cite{GVAE}, the KL term regularizes the latent space for reconstruction. 
In contrast, our KL term serves as a cross-domain alignment regularizer: by minimizing $\mathrm{KL}(q_\phi(\mathbf{Z}^i \mid \hat{\mathbf{X}}^i) \,\|\, p(\mathbf{Z}))$ for all domain $i$, we force each domain to adapt its posterior distribution to the same target. 
Combined with the shared encoder and DPU-aligned inputs, this mechanism enables the model to learn generalizable semantic representations across domains.

\section{Experiments}
\label{eq:sec_exp}

In this section, we comprehensively evaluate the effectiveness and generalizability of the proposed LEDA by conducting comparisons with multiple baseline methods on both node-level and graph-level tasks. In Section \ref{sec:Experimental Settings}, we introduce the experimental setup (including datasets, baselines and setups), while in Section \ref{sec:performance_analysis}, we analyze the model's performance across various tasks. Furthermore, in Section \ref{sec:model_analysis}, we present ablation studies to examine the effectiveness of each component of LEDA.
We also provide more detailed experimental results in the appendix.
\subsection{Experimental Settings}
\label{sec:Experimental Settings}

\textbf{Datsets.}
To effectively validate the performance of  LEDA, we conduct experiments on eleven widely used benchmark datasets. 
We consider graph datasets spanning various domains, including: 1) \textit{citation networks}: Cora, CiteSeer, and PubMed \cite{citationnetwork, Cora}; 2) \textit{co-purchase networks}: Photo and Computers \cite{PhotoComputers}; 3)\textit{co-author networks}: CS and Physics \cite{CSPhysics}; 4) \textit{social network}: COLLAB, IMDB-BINARY \cite{socialnetworks}; 5) \textit{biological network}: PROTEINS \cite{PROTEINS}, ENZYMES \cite{ENZYMES}. 

\textbf{Baselines.}
We compare LEDA's performance in four main categories as outlined below: 1) \textit{Non-parametric models}: Raw features, SVD \cite{SVD} and DeepWalk \cite{deepwalk}; 2) \textit{Graph self-supervised models}: GRACE \cite{GRACE}, GCA \cite{GCA}, BGRL \cite{BGRL}, DGI \cite{DGI}, GraphMAE \cite{GraphMAE}; 3) \textit{Graph pre-train models}: 
OFA \cite{OFA}, GraphControl \cite{graphcontrol}, GPPT \cite{GPPT}, GraphPrompt \cite{GraphPrompt}, GPF \cite{GPF}, GCOPE \cite{GCOPE}, MDGPT \cite{MDGPT}, FUG \cite{FUG}, All-in-one \cite{All-in-One}, GPF-Plus \cite{GPF}, ULTRA(3g) \cite{ultra}, SCORE \cite{SCORE}; 4) \textit{Graph semi-supervised models}: GCN \cite{GCN}.

\textbf{Setups.} To evaluate the performance of LEDA, we conduct extensive experiments under various settings as following: 1) \textit{Cross-domain node classification}; 2) \textit{Cross-domain few-shot node classification}; 3) \textit{Cross-domain zero-shot graph classification}. More implementation details can be found in Appendix \ref{app_implementation_detail}.

\begin{table*}[h]
\centering

\caption{Performance on cross-domain 1-shot node classification. Methods with * are reported from \cite{MDGPT}.  Each column represents a test domain, while others are train domains. The best results are bolded and the second-best are underlined.}
\resizebox{1.0\textwidth}{!}{
\begin{tabular}{l|ccccccc}
\toprule
\textbf{Method} \textbackslash \textbf{Test data}  & \textbf{Cora} & \textbf{CiteSeer} & \textbf{PubMed} & \textbf{Photo} & \textbf{Computers} & \textbf{CS} \\
\midrule
DGI  & 24.19 $\pm$ 5.43 & 27.66 $\pm$ 4.93 & 45.77 $\pm$ 7.16  & 44.92 $\pm$ 7.71  & 27.92 $\pm$ 5.35 & 67.00 $\pm$ 6.72 \\
BGRL  & 35.29 $\pm$ 6.54 & 34.99 $\pm$ 6.92 & 43.56 $\pm$ 7.54  & 48.25 $\pm$ 6.30 & 35.84 $\pm$ 8.74 & 62.55 $\pm$ 4.96 \\
GraphMAE &35.99 $\pm$ 7.81 &38.35 $\pm$ 8.36 &49.23 $\pm$ 8.36 & 55.99 $\pm$ 9.62 & 45.17 $\pm$ 10.91 & 66.64 $\pm$ 7.03\\ 
\midrule
GPPT*  & 15.37 $\pm$ 4.51 & 23.24 $\pm$ 2.94 & 36.56 $\pm$ 5.31  & 16.19 $\pm$ 4.73 & 19.22 $\pm$ 8.71  &- \\
GraphPrompt* & 35.90 $\pm$ 7.10 & 32.76 $\pm$ 7.66 & 43.34 $\pm$ 10.66  & 49.88 $\pm$ 8.31 & 43.03 $\pm$ 10.35 &- \\
GPF*  & 37.84 $\pm$ 11.07 & 37.61 $\pm$ 8.87 & 46.36 $\pm$ 7.48  & 49.42 $\pm$ 7.04 & 37.00 $\pm$ 6.52 &- \\
GCOPE*  & 33.38 $\pm$ 6.86 & 35.56 $\pm$ 6.81 & 42.10 $\pm$ 8.07  & 48.52 $\pm$ 7.78 & 40.22 $\pm$ 7.82 &- \\
MDGPT*  & 42.26 $\pm$ 10.18 & \underline{42.40 $\pm$ 9.26} & 49.82 $\pm$ 8.38  & \textbf{64.82 $\pm$ 10.53} & \underline{49.77 $\pm$ 11.00} &- \\
\midrule
FUG  & \underline{42.92 $\pm$ 11.51} &36.78 $\pm$ 7.46 & \underline{53.83 $\pm$ 8.77} & 62.80 $\pm$ 10.57 & 47.92 $\pm$ 11.51 & \underline{68.32 $\pm$ 6.55}\\
\textbf{LEDA (Ours)} & \textbf{50.70 $\pm$ 10.67} & \textbf{43.83 $\pm$ 10.04} & \textbf{54.34 $\pm$ 11.08}  & \underline{64.35 $\pm$ 9.46} & \textbf{51.00 $\pm$ 12.19} & \textbf{68.74 $\pm$ 8.00} \\
\bottomrule
\end{tabular}
}
\label{tab_one_shot}
\end{table*}

\subsection{Performance Analysis}
\label{sec:performance_analysis}
\textbf{Cross-domain node classification.} To evaluate the effectiveness of LEDA in cross-domain node classification scenarios, we follow the setup of \cite{FUG} and adopt a fixed set of hyper-parameters that are kept consistent across all training datasets. 
It is noteworthy that LEDA outperforms all in-domain baselines in cross-domain testing scenarios, indicating its strong ability to learn broad and transferable knowledge from multi-domain graphs and generalize to unseen domains.
Moreover, while OFA and GraphControl are cross-domain methods capable of training on multi-domain graphs, their performance in cross-domain scenarios remains limited compared to LEDA, which can be attributed to the potential loss of critical information during feature alignment. 
OFA textualizes node attributes and leverages a large language model for encoding, whereas GraphControl employs kernel-based similarity to align node representations via a feature-driven adjacency matrix. 
Such data alignment strategies are sub-optimal when applied to highly complex features. Furthermore, recent advanced graph pre-training method FUG, despite designing a dimension encoder to learn knowledge lossless, still falls short in modeling shared semantics from multi-domain graphs. 
FUG employs a universal contrastive loss to constrain node distributions in representation space, which works well in in-domain  settings but may become suboptimal for multi-domain knowledge learning. This further supports our analysis in Section \ref{sec:analysis_of_ugp}.
As shown in Table \ref{tab:multi_domain_10test}, LEDA outperforms FUG in cross-domain settings across six datasets, with particularly notable improvements on the CiteSeer dataset.
This stems from the avoidance of modeling relative relationships among multi-domain graphs with entangled or inconsistent semantics.
To further validate the transferability of LEDA, we also conducted experiments where pre-training was performed on the citation networks and testing was carried out on the co-purchase networks, and vice versa. The experimental results are presented in Appendix \ref{app_supplementary_results}.

\begin{figure*}[h]

\centering
\subfigure{
\resizebox{0.31\linewidth}{!}{\includegraphics{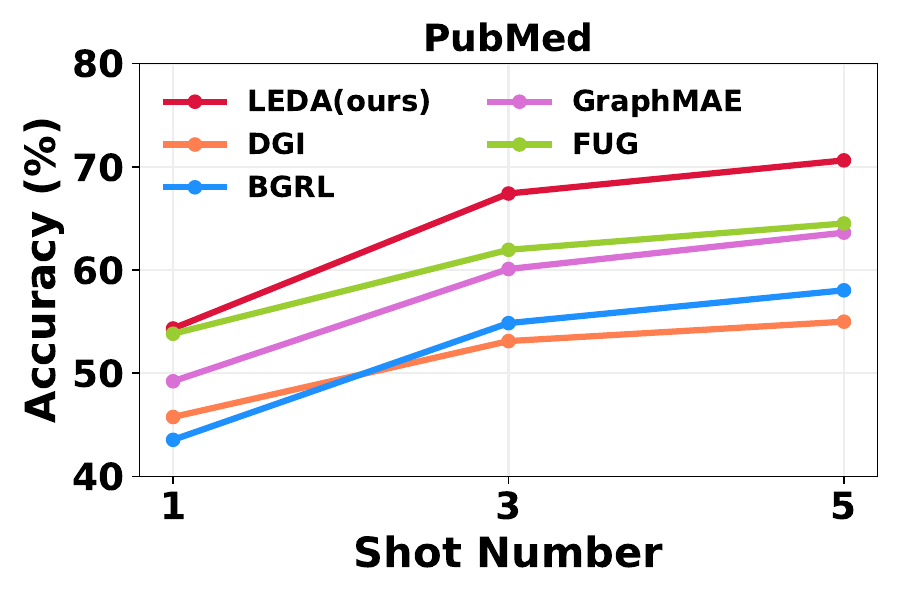}}
}
\subfigure{
\resizebox{0.31\linewidth}{!}{\includegraphics{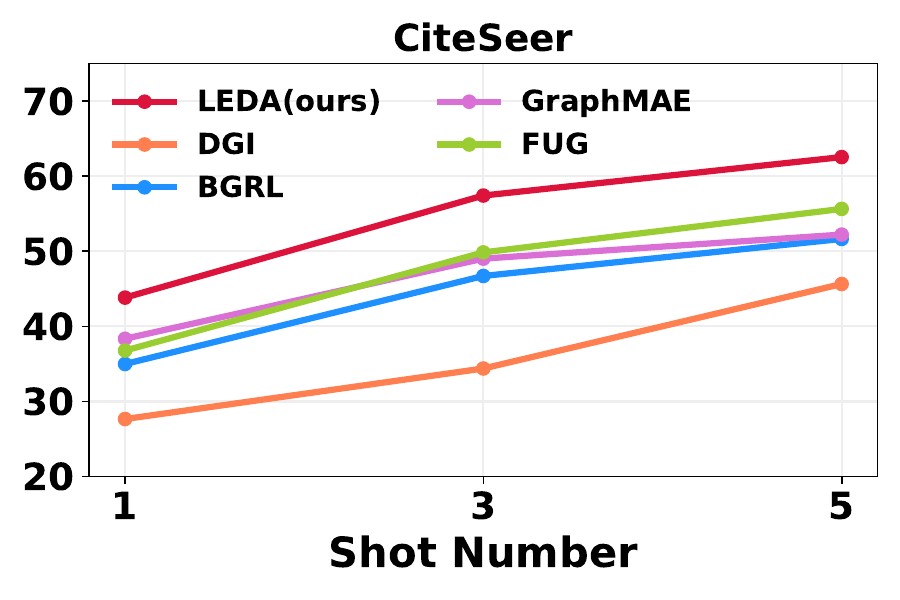}}
}
\subfigure{
\resizebox{0.31\linewidth}{!}{\includegraphics{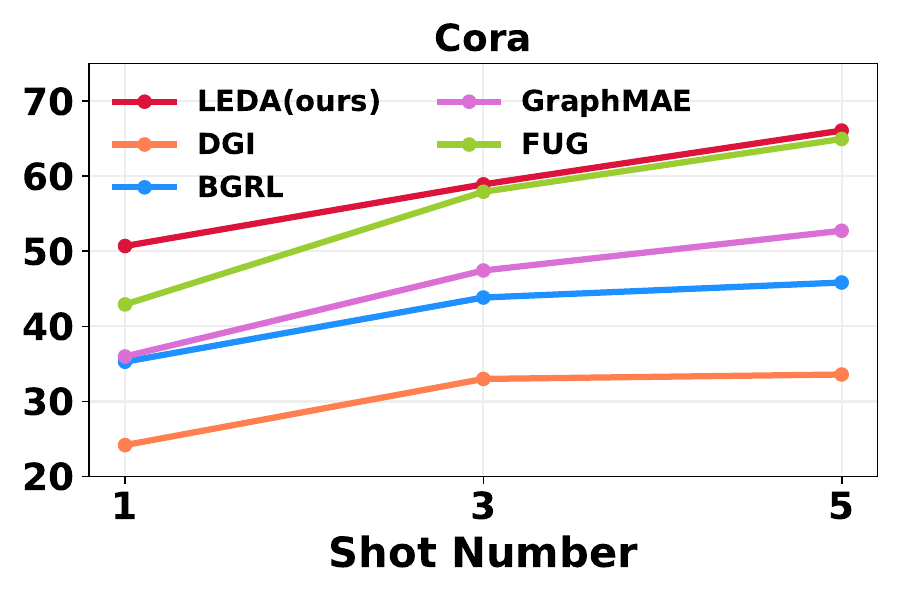}}
}
\caption{Performance on cross-domain few-shot node classification. The red line denotes our method. For traditional in-domain methods, we simply unify the input data dimensions by SVD. }
\label{fig:few_shot}
\end{figure*}

\textbf{Cross-domain few-shot node classification.} In this scenario, we use fixed-parameters to pre-train the model and perform no fine-tuning or prompt-tuning. Specifically, we randomly select $k$ nodes ($k$-shot) from each class and calculate the mean of their vectors  as the class prototype vector and predict the node's class by calculating the similarity between the model output and the prototype vector. 
Unlike the above cross-domain node classification, the 1-shot setting places a stronger demand on the model's generalization ability. 
When $k$ is set to 1,it can be observed from Table \ref{tab_one_shot} that LEDA achieves the best results on five datasets and the second-best results on Photo. Notably, methods with $*$ incorporate prompt-tuning during downstream tasks, which are not currently used in LEDA.
Moreover, FUG, which does not employ any fine-tuning or prompt-tuning strategies, performs well on most datasets; however, it shows suboptimal results on some datasets.
This is due to its reliance on the quality of the sampled data. 
In contrast, LEDA does not require sampling from multi-domain data and achieves excellent performance across all datasets. 
Besides, we also conduct experiments with $k$ set to 3 and 5, as illustrated in Figure \ref{fig:few_shot}. We observe that LEDA achieves the best performance in all datasets, demonstrating its exceptional generalization ability in few-shot scenarios. More experimental results in $k$-shot scenario can be found in Appendix \ref{app_supplementary_results}.

\begin{table*}[h]
\centering

\caption{Performance on zero-shot graph classification. Methods with * are reported from \cite{SCORE}. The best results are bolded and the second-best are underline.}
\resizebox{1.0\textwidth}{!}{
\begin{tabular}{l|cccccccc}
\toprule
\multicolumn{1}{l|}{\multirow{2}{*}{\textbf{Method}}} & \multicolumn{2}{c}{\textbf{IMDB-B }} &  \multicolumn{2}{c}{\textbf{PROTEINS}}  &  \multicolumn{2}{c}{\textbf{ENZYMES}} &  \multicolumn{2}{c}{\textbf{COLLAB}} \\
\cmidrule{2-9}
& \multicolumn{1}{c}{\textbf{Accuracy} } & \multicolumn{1}{c}{\textbf{Macro-F1}} & \multicolumn{1}{c}{\textbf{Accuracy} } & \multicolumn{1}{c}{\textbf{Macro-F1}} & \multicolumn{1}{c}{\textbf{Accuracy} } & \multicolumn{1}{c}{\textbf{Macro-F1}}  & \multicolumn{1}{c}{\textbf{Accuracy} } & \multicolumn{1}{c}{\textbf{Macro-F1}} \\ 
\midrule

GPPT \dag* & 50.15 $\pm$ 0.75 & 44.16  $\pm$ 6.70 & 60.92 $\pm$ 2.47 & 47.07 $\pm$ 11.95 & 21.29 $\pm$ 3.79 & 19.87 $\pm$ 2.99 & 47.18 $\pm$ 5.93 & 42.87 $\pm$ 7.70 \\ 
All-in-one \dag* & 60.07 $\pm$ 4.81 & 56.88  $\pm$ 0.80 & 66.49 $\pm$ 6.26 & 64.68 $\pm$ 5.35 & 23.96 $\pm$ 1.45 & 19.66 $\pm$ 3.11 & 51.66 $\pm$ 0.26 & 47.78 $\pm$ 0.10 \\
GraphPrompt \dag* & 54.75 $\pm$ 12.43 & 52.10  $\pm$ 13.61 & 59.17 $\pm$ 11.26 & 58.30 $\pm$ 10.88 & 22.29 $\pm$ 3.50 & 19.52 $\pm$ 3.36 & 48.25 $\pm$ 13.64 & 43.35 $\pm$ 10.75 \\
GPF \dag* & 59.65 $\pm$ 5.06 & 56.22  $\pm$ 6.17 & 63.91 $\pm$ 3.26 & 57.01 $\pm$ 5.79 & 22.00 $\pm$ 1.25 &17.34 $\pm$ 2.45 & 47.42 $\pm$ 11.22 & 38.14 $\pm$ 0.44 \\
GPF-Plus \dag*  & 57.93 $\pm$ 1.62 & 55.55  $\pm$ 2.03 & 62.92 $\pm$ 2.78 & 57.58 $\pm$ 7.28 & 22.92 $\pm$ 1.64 & 18.39 $\pm$ 2.76 & 47.24 $\pm$ 0.29 & 41.24 $\pm$ 0.31 \\   
\midrule
ULTRA(3g) \ddag*  & 49.25 $\pm$ 0.00 & 38.87  $\pm$ 0.00 & 58.09 $\pm$ 0.00 & 37.48 $\pm$ 0.00 & 15.21 $\pm$ 0.00 & 5.84 $\pm$ 0.00 & \underline{64.53 $\pm$ 0.00} & \underline{55.36 $\pm$ 0.00} \\
SCORE  \ddag*   & \underline{61.83 $\pm$ 1.60} & \underline{60.91  $\pm$ 2.18} & \underline{68.54 $\pm$ 1.47} & \underline{65.23 $\pm$ 1.37}& \underline{22.92 $\pm$ 2.03} & \underline{21.77 $\pm$ 2.17} & \textbf{65.45 $\pm$ 1.05} & \textbf{57.71 $\pm$ 1.82}  \\
\textbf{LEDA (Ours) \ddag}  & \textbf{70.64 $\pm$ 1.91} & \textbf{70.60 $\pm$ 1.91} & \textbf{82.85 $\pm$ 1.02} & \textbf{82.67 $\pm$ 1.07} & \textbf{45.98 $\pm$ 2.03} & \textbf{40.07 $\pm$ 2.33} & 55.83 $\pm$ 1.19 & 51.28 $\pm$ 1.59  \\
\midrule
GCN* & 57.30 $\pm$ 0.98 & 54.62  $\pm$ 1.12 & 56.36 $\pm$ 7.97 & 46.69 $\pm$ 10.82 & 20.58 $\pm$ 2.00 & 15.25 $\pm$ 3.96 & 47.23 $\pm$ 0.61 & 41.10 $\pm$ 0.39 \\
\bottomrule
\end{tabular}
}
\label{tab_graph_classification}
\end{table*}
\textbf{Cross-domain zero-shot graph classification.} To enable a broader and fairer comparison of LEDA's performance in graph classification, we follow the evaluation setup from \cite{SCORE}.
Methods with $\dag$ refer to graph prompt-based algorithms under the 1-shot setting, while methods with $\ddag$ represent algorithms under the 0-shot setting, without fine-tuning or prompt tuning.
As shown in Table \ref{tab_graph_classification}, LEDA significantly outperforms baselines on most datasets, demonstrating its ability to learn effective knowledge from multi-domain graphs.
Notably, on the ENZYMES dataset, LEDA surpasses the second-best method by 23.06\% in Accuracy, further highlighting its strong generalization capability.
Moreover, it is worth noting that LEDA achieves moderate performance on COLLAB,in contrast to its consistently strong results in node classification.
This may be due to COLLAB's comparatively denser topology relative to the other three datasets, as shown in Table \ref{tab:datasets}.

\subsection{Model Analysis}
\label{sec:model_analysis}

\begin{table}
 \caption{Ablation studies on three cross-domain one-shot node classification datasets. CL means Contrastive Learning. The best result is bolded. \\}

\begin{tabular}{lccc}
\toprule
\textbf{Method} & \textbf{Cora} &\textbf{CiteSeer} & \textbf{PubMed} \\
\midrule
w/o DPU &30.87 $\pm$ 7.47 & 36.55 $\pm$ 7.83& 48.50 $\pm$ 9.23\\
w/o LDA 
& 34.55 $\pm$ 7.24
& 41.07 $\pm$ 7.60
& 52.87 $\pm$ 8.26
\\
DPU+CL
& 34.08 $\pm$ 7.12
& 40.85 $\pm$ 7.59
& 52.44 $\pm$ 8.30
\\
\textbf{LEDA} & \textbf{50.70 $\pm$ 10.67} & \textbf{43.83 $\pm$ 10.04} & \textbf{54.34 $\pm$ 11.08} \\
 \bottomrule
 
\end{tabular}
\label{tab:ablation}
\end{table}

To evaluate each component in LEDA, we conducted ablation studies on three datasets.
We first removed DPU and replaced it with SVD-based data alignment approach.
As shown in Table \ref{tab:ablation},  LEDA without DPU results in a significant performance drop across all datasets.
This supports our claim that while matrix factorization techniques can align the dimensions of different graph data, they may introduce redundant information or discard essential semantics, hindering the learning of universal knowledge from multi-domain graphs.
We then removed LDA, and similarly, it also resulted in varying degrees of performance degradation across all datasets. 
This indicates that although DPU preserves as much information as possible from different graph datasets, its lack of cross-domain universal knowledge extraction leads to suboptimal performance.
Furthermore, building on the removal of the LDA module, we introduced a contrastive loss to enhance the discriminability of representations in the unified space. 
Here, we use an InfoNCE-based contrastive loss, which aims to push anchor nodes away from negative samples. 
It is challenging to accurately define negative samples in the multi-domain pre-training setting; therefore, we adopt the mean feature of all nodes as the negative samples.
By comparing DPU+CL with w/o LDA, we observe that the model performance remains nearly unchanged or slightly degrades after introducing contrastive loss.
This aligns with our discussion in Section \ref{sec:analysis_of_ugp}, suggesting that contrastive loss may lead to sub-optimal performance in multi-domain scenarios.
Additional experimental results and detailed explanations are provided in Appendix \ref{app_supplementary_results}.

\section{Conclusion}
In this paper, we propose a novel universal graph pre-training model, LEDA, which effectively learns transferable knowledge from multi-domain graphs. 
By introducing a learnable domain projection unit, LEDA adaptively aligns features from multiple domains into a unified embedding space while preserving essential semantics and avoiding semantic conflicts. 
Furthermore, LEDA aligns the posterior distributions encoded from different domains with a shared prior distribution, enabling effective universal knowledge learning. Additionally, during distribution alignment, LEDA further guides the base vectors of domain-specific projections toward the shared semantic direction, thereby facilitating the learning of stable cross-domain representations. Extensive experimental results on multiple datasets and downstream tasks demonstrate the superior performance of LEDA. In few-shot cross-domain settings, LEDA outperforms existing methods by a considerable margin.

\begin{acks}
This work was supported by the National Natural Science Foundation of China (No. 62422210, and No. 62276187), the National Key Research and Development Program of China (No. 2023YFC3304503), the Hong Kong RGC theme-based Strategic Target Grant Scheme (STG STG1/M-501/23-N), the Hong Kong Global STEM Professor Scheme, the Hong Kong Jockey Club Charities Trust and the Natural Science Foundation of Jilin Province (Grant No. 20250102211JC).
\end{acks}

\bibliographystyle{ACM-Reference-Format}
\balance
\bibliography{slz}

@article{FUG,
  title={FUG: Feature-universal graph contrastive pre-training for graphs with diverse node features},
  author={Zhao, Jitao and Jin, Di and Ge, Meng and Shan, Lianze and Wang, Xin and He, Dongxiao and Feng, Zhiyong},
  journal={Advances in Neural Information Processing Systems},
  volume={37},
  pages={4003--4034},
  year={2024}
}

@inproceedings{zeroG,
  title={Zerog: Investigating cross-dataset zero-shot transferability in graphs},
  author={Li, Yuhan and Wang, Peisong and Li, Zhixun and Yu, Jeffrey Xu and Li, Jia},
  booktitle={Proceedings of the 30th ACM SIGKDD Conference on Knowledge Discovery and Data Mining},
  pages={1725--1735},
  year={2024}
}

@article{unilink,
  title={Universal Link Predictor By In-Context Learning on Graphs},
  author={Dong, Kaiwen and Mao, Haitao and Guo, Zhichun and Chawla, Nitesh V},
  journal={arXiv preprint arXiv:2402.07738},
  year={2024}
}

@article{ultra,
  title={Towards foundation models for knowledge graph reasoning},
  author={Galkin, Mikhail and Yuan, Xinyu and Mostafa, Hesham and Tang, Jian and Zhu, Zhaocheng},
  journal={arXiv preprint arXiv:2310.04562},
  year={2023}
}

@article{OFA,
  title={One for all: Towards training one graph model for all classification tasks},
  author={Liu, Hao and Feng, Jiarui and Kong, Lecheng and Liang, Ningyue and Tao, Dacheng and Chen, Yixin and Zhang, Muhan},
  journal={arXiv preprint arXiv:2310.00149},
  year={2023}
}

@inproceedings{graphcontrol,
  title={Graphcontrol: Adding conditional control to universal graph pre-trained models for graph domain transfer learning},
  author={Zhu, Yun and Wang, Yaoke and Shi, Haizhou and Zhang, Zhenshuo and Jiao, Dian and Tang, Siliang},
  booktitle={Proceedings of the ACM Web Conference 2024},
  pages={539--550},
  year={2024}
}

@inproceedings{gcc,
  title={Gcc: Graph contrastive coding for graph neural network pre-training},
  author={Qiu, Jiezhong and Chen, Qibin and Dong, Yuxiao and Zhang, Jing and Yang, Hongxia and Ding, Ming and Wang, Kuansan and Tang, Jie},
  booktitle={Proceedings of the 26th ACM SIGKDD international conference on knowledge discovery \& data mining},
  pages={1150--1160},
  year={2020}
}

@article{achiam2023gpt,
  title={Gpt-4 technical report},
  author={Achiam, Josh and Adler, Steven and Agarwal, Sandhini and Ahmad, Lama and Akkaya, Ilge and Aleman, Florencia Leoni and Almeida, Diogo and Altenschmidt, Janko and Altman, Sam and Anadkat, Shyamal and others},
  journal={arXiv preprint arXiv:2303.08774},
  year={2023}
}

@inproceedings{bert,
  title={Bert: Pre-training of deep bidirectional transformers for language understanding},
  author={Devlin, Jacob and Chang, Ming-Wei and Lee, Kenton and Toutanova, Kristina},
  booktitle={Proceedings of the 2019 conference of the North American chapter of the association for computational linguistics: human language technologies, volume 1 (long and short papers)},
  pages={4171--4186},
  year={2019}
}

@inproceedings{samgpt,
  title={SAMGPT: Text-free graph foundation model for multi-domain pre-training and cross-domain adaptation},
  author={Yu, Xingtong and Gong, Zechuan and Zhou, Chang and Fang, Yuan and Zhang, Hui},
  booktitle={Proceedings of the ACM on Web Conference 2025},
  pages={1142--1153},
  year={2025}
}

@article{molecular_analysis,
  title={ModuleDiscoverer: Identification of regulatory modules in protein-protein interaction networks},
  author={Vlaic, Sebastian and Conrad, Theresia and Tokarski-Schnelle, Christian and Gustafsson, Mika and Dahmen, Uta and Guthke, Reinhard and Schuster, Stefan},
  journal={Scientific reports},
  volume={8},
  number={1},
  pages={433},
  year={2018},
  publisher={Nature Publishing Group UK London}
}

@article{citationnetwork,
  title={Collective classification in network data},
  author={Sen, Prithviraj and Namata, Galileo and Bilgic, Mustafa and Getoor, Lise and Galligher, Brian and Eliassi-Rad, Tina},
  journal={AI magazine},
  volume={29},
  number={3},
  pages={93--93},
  year={2008}
}

@inproceedings{Cora,
  title={Revisiting semi-supervised learning with graph embeddings},
  author={Yang, Zhilin and Cohen, William and Salakhudinov, Ruslan},
  booktitle={International conference on machine learning},
  pages={40--48},
  year={2016},
  organization={PMLR}
}

@inproceedings{PhotoComputers,
  title={Image-based recommendations on styles and substitutes},
  author={McAuley, Julian and Targett, Christopher and Shi, Qinfeng and Van Den Hengel, Anton},
  booktitle={Proceedings of the 38th international ACM SIGIR conference on research and development in information retrieval},
  pages={43--52},
  year={2015}
}

@inproceedings{CSPhysics,
  title={An overview of microsoft academic service (mas) and applications},
  author={Sinha, Arnab and Shen, Zhihong and Song, Yang and Ma, Hao and Eide, Darrin and Hsu, Bo-June and Wang, Kuansan},
  booktitle={Proceedings of the 24th international conference on world wide web},
  pages={243--246},
  year={2015}
}

@inproceedings{socialnetworks,
  title={Deep graph kernels},
  author={Yanardag, Pinar and Vishwanathan, SVN},
  booktitle={Proceedings of the 21th ACM SIGKDD international conference on knowledge discovery and data mining},
  pages={1365--1374},
  year={2015}
}

@article{PROTEINS,
  title={Distinguishing enzyme structures from non-enzymes without alignments},
  author={Dobson, Paul D and Doig, Andrew J},
  journal={Journal of molecular biology},
  volume={330},
  number={4},
  pages={771--783},
  year={2003},
  publisher={Elsevier}
}

@article{SVD,
  title={On the early history of the singular value decomposition},
  author={Stewart, Gilbert W},
  journal={SIAM review},
  volume={35},
  number={4},
  pages={551--566},
  year={1993},
  publisher={SIAM}
}

@inproceedings{GPPT,
  title={Gppt: Graph pre-training and prompt tuning to generalize graph neural networks},
  author={Sun, Mingchen and Zhou, Kaixiong and He, Xin and Wang, Ying and Wang, Xin},
  booktitle={Proceedings of the 28th ACM SIGKDD Conference on Knowledge Discovery and Data Mining},
  pages={1717--1727},
  year={2022}
}

@inproceedings{GraphPrompt,
  title={Graphprompt: Unifying pre-training and downstream tasks for graph neural networks},
  author={Liu, Zemin and Yu, Xingtong and Fang, Yuan and Zhang, Xinming},
  booktitle={Proceedings of the ACM web conference 2023},
  pages={417--428},
  year={2023}
}

@article{GPF,
  title={Universal prompt tuning for graph neural networks},
  author={Fang, Taoran and Zhang, Yunchao and Yang, Yang and Wang, Chunping and Chen, Lei},
  journal={Advances in Neural Information Processing Systems},
  volume={36},
  pages={52464--52489},
  year={2023}
}

@inproceedings{GCOPE,
  title={All in one and one for all: A simple yet effective method towards cross-domain graph pretraining},
  author={Zhao, Haihong and Chen, Aochuan and Sun, Xiangguo and Cheng, Hong and Li, Jia},
  booktitle={Proceedings of the 30th ACM SIGKDD Conference on Knowledge Discovery and Data Mining},
  pages={4443--4454},
  year={2024}
}

@article{MDGPT,
  title={Text-free multi-domain graph pre-training: Toward graph foundation models},
  author={Yu, Xingtong and Zhou, Chang and Fang, Yuan and Zhang, Xinming},
  journal={arXiv preprint arXiv:2405.13934},
  year={2024}
}

@article{GRACE,
  author       = {Yanqiao Zhu and
                  Yichen Xu and
                  Feng Yu and
                  Qiang Liu and
                  Shu Wu and
                  Liang Wang},
  title        = {Deep Graph Contrastive Representation Learning},
  journal      = {CoRR},
  volume       = {abs/2006.04131},
  year         = {2020},
  url          = {https://arxiv.org/abs/2006.04131},
  eprinttype    = {arXiv},
  eprint       = {2006.04131},
  bibsource    = {dblp computer science bibliography, https://dblp.org}
}

@inproceedings{DGI,
  author       = {Petar Velickovic and
                  William Fedus and
                  William L. Hamilton and
                  Pietro Li{\`{o}} and
                  Yoshua Bengio and
                  R. Devon Hjelm},
  title        = {Deep Graph Infomax},
  booktitle    = {7th International Conference on Learning Representations, {ICLR} 2019,
                  New Orleans, LA, USA, May 6-9, 2019},
  publisher    = {OpenReview.net},
  year         = {2019},
  url          = {https://openreview.net/forum?id=rklz9iAcKQ},
  bibsource    = {dblp computer science bibliography, https://dblp.org}
}

@inproceedings{BGRL,
  title={Bootstrapped representation learning on graphs},
  author={Thakoor, Shantanu and Tallec, Corentin and Azar, Mohammad Gheshlaghi and Munos, R{\'e}mi and Veli{\v{c}}kovi{\'c}, Petar and Valko, Michal},
  booktitle={ICLR 2021 Workshop on Geometrical and Topological Representation Learning},
  year={2021}
}

@article{GVAE,
  author       = {Thomas N. Kipf and
                  Max Welling},
  title        = {Variational Graph Auto-Encoders},
  journal      = {CoRR},
  volume       = {abs/1611.07308},
  year         = {2016},
  url          = {http://arxiv.org/abs/1611.07308},
  eprinttype    = {arXiv},
  eprint       = {1611.07308},
  bibsource    = {dblp computer science bibliography, https://dblp.org}
}

@inproceedings{GCN,
  author       = {Thomas N. Kipf and
                  Max Welling},
  title        = {Semi-Supervised Classification with Graph Convolutional Networks},
  booktitle    = {5th International Conference on Learning Representations, {ICLR} 2017,
                  Toulon, France, April 24-26, 2017, Conference Track Proceedings},
  publisher    = {OpenReview.net},
  year         = {2017},
  url          = {https://openreview.net/forum?id=SJU4ayYgl},
  bibsource    = {dblp computer science bibliography, https://dblp.org}
}

@inproceedings{GAT,
  author       = {Petar Velickovic and
                  Guillem Cucurull and
                  Arantxa Casanova and
                  Adriana Romero and
                  Pietro Li{\`{o}} and
                  Yoshua Bengio},
  title        = {Graph Attention Networks},
  booktitle    = {6th International Conference on Learning Representations, {ICLR} 2018,
                  Vancouver, BC, Canada, April 30 - May 3, 2018, Conference Track Proceedings},
  publisher    = {OpenReview.net},
  year         = {2018},
  url          = {https://openreview.net/forum?id=rJXMpikCZ},
  bibsource    = {dblp computer science bibliography, https://dblp.org}
}

@inproceedings{GraphSAGE,
  author       = {William L. Hamilton and
                  Zhitao Ying and
                  Jure Leskovec},
  title        = {Inductive Representation Learning on Large Graphs},
  booktitle    = {Advances in Neural Information Processing Systems 30: Annual Conference
                  on Neural Information Processing Systems,
                  Long Beach, CA, {USA}},
  pages        = {1024--1034},
  year         = {2017},
  bibsource    = {dblp computer science bibliography, https://dblp.org}
}

@article{InfoNCE,
  title={Representation learning with contrastive predictive coding},
  author={Oord, Aaron van den and Li, Yazhe and Vinyals, Oriol},
  journal={arXiv preprint arXiv:1807.03748},
  year={2018}
}

@inproceedings{GCA,
  author       = {Yanqiao Zhu and
                  Yichen Xu and
                  Feng Yu and
                  Qiang Liu and
                  Shu Wu and
                  Liang Wang},
  title        = {Graph Contrastive Learning with Adaptive Augmentation},
  booktitle    = {{WWW} '21: The Web Conference 2021, Virtual Event / Ljubljana, Slovenia,
                  April 19-23, 2021},
  pages        = {2069--2080},
  publisher    = {{ACM} / {IW3C2}},
  year         = {2021},
  url          = {https://doi.org/10.1145/3442381.3449802},
  bibsource    = {dblp computer science bibliography, https://dblp.org}
}

@article{GNNsurvey1,
  title={A comprehensive survey on graph neural networks},
  author={Wu, Zonghan and Pan, Shirui and Chen, Fengwen and Long, Guodong and Zhang, Chengqi and Philip, S Yu},
  journal      = {{IEEE} Trans. Neural Networks Learn. Syst.},
  volume       = {32},
  number       = {1},
  pages        = {4--24},
  year         = {2021},
  url          = {https://doi.org/10.1109/TNNLS.2020.2978386},
  bibsource    = {dblp computer science bibliography, https://dblp.org}
}

@article{GNNsurvey2,
  title={Graph neural networks: A review of methods and applications},
  author={Zhou, Jie and Cui, Ganqu and Hu, Shengding and Zhang, Zhengyan and Yang, Cheng and Liu, Zhiyuan and Wang, Lifeng and Li, Changcheng and Sun, Maosong},
  journal={AI open},
  volume={1},
  pages={57--81},
  year={2020},
  publisher={Elsevier}
}

@article{adamw,
  title={Decoupled weight decay regularization},
  author={Loshchilov, Ilya and Hutter, Frank},
  journal={arXiv preprint arXiv:1711.05101},
  year={2017}
}

@inproceedings{deepwalk,
  title={Deepwalk: Online learning of social representations},
  author={Perozzi, Bryan and Al-Rfou, Rami and Skiena, Steven},
  booktitle={The 20th {ACM} {SIGKDD} International Conference on Knowledge Discovery
             and Data Mining, {KDD} '14, New York, NY, {USA} - August 24 - 27, 2014},
  pages={701--710},
  publisher={{ACM}},
  year={2014},
  url={https://doi.org/10.1145/2623330.2623732}
}

@inproceedings{MaskGAE,
  author       = {Jintang Li and
                  Ruofan Wu and
                  Wangbin Sun and
                  Liang Chen and
                  Sheng Tian and
                  Liang Zhu and
                  Changhua Meng and
                  Zibin Zheng and
                  Weiqiang Wang},
  title        = {What's Behind the Mask: Understanding Masked Graph Modeling for Graph
                  Autoencoders},
  booktitle    = {Proceedings of the 29th {ACM} {SIGKDD} Conference on Knowledge Discovery
                  and Data Mining, {KDD} 2023, Long Beach, CA, USA, August 6-10, 2023},
  pages        = {1268--1279},
  publisher    = {{ACM}},
  year         = {2023},
  url          = {https://doi.org/10.1145/3580305.3599546},
  bibsource    = {dblp computer science bibliography, https://dblp.org}
}

@inproceedings{GraphMAE,
  author       = {Zhenyu Hou and
                  Xiao Liu and
                  Yukuo Cen and
                  Yuxiao Dong and
                  Hongxia Yang and
                  Chunjie Wang and
                  Jie Tang},
  title        = {GraphMAE: Self-Supervised Masked Graph Autoencoders},
  booktitle    = {{KDD} '22: The 28th {ACM} {SIGKDD} Conference on Knowledge Discovery
                  and Data Mining, Washington, DC, USA, August 14 - 18, 2022},
  pages        = {594--604},
  publisher    = {{ACM}},
  year         = {2022},
  url          = {https://doi.org/10.1145/3534678.3539321},
  bibsource    = {dblp computer science bibliography, https://dblp.org}
}

@article{bioinformatics,
  title={Graph neural networks and their current applications in bioinformatics},
  author={Zhang, Xiao-Meng and Liang, Li and Liu, Lin and Tang, Ming-Jing},
  journal={Frontiers in genetics},
  volume={12},
  pages={690049},
  year={2021},
  publisher={Frontiers Media SA}
}

@inproceedings{LightGCN,
  author       = {Xiangnan He and
                  Kuan Deng and
                  Xiang Wang and
                  Yan Li and
                  Yong{-}Dong Zhang and
                  Meng Wang},
  editor       = {Jimmy X. Huang and
                  Yi Chang and
                  Xueqi Cheng and
                  Jaap Kamps and
                  Vanessa Murdock and
                  Ji{-}Rong Wen and
                  Yiqun Liu},
  title        = {LightGCN: Simplifying and Powering Graph Convolution Network for Recommendation},
  booktitle    = {Proceedings of the 43rd International {ACM} {SIGIR} conference on
                  research and development in Information Retrieval, {SIGIR} 2020, Virtual
                  Event, China, July 25-30, 2020},
  pages        = {639--648},
  publisher    = {{ACM}},
  year         = {2020},
  url          = {https://doi.org/10.1145/3397271.3401063},
  doi          = {10.1145/3397271.3401063},
  bibsource    = {dblp computer science bibliography, https://dblp.org}
}

@article{Matrix_Factorization ,
  title={An Introduction to Matrix factorization and Factorization Machines in Recommendation System, and Beyond},
  author={Zhang, Yuefeng},
  journal={arXiv preprint arXiv:2203.11026},
  year={2022}
}

@book{randomwalk,
  title={Random walk: a modern introduction},
  author={Lawler, Gregory F and Limic, Vlada},
  volume={123},
  year={2010},
  publisher={Cambridge University Press}
}

@inproceedings{social_network_2,
  title={Consisrec: Enhancing gnn for social recommendation via consistent neighbor aggregation},
  author={Yang, Liangwei and Liu, Zhiwei and Dou, Yingtong and Ma, Jing and Yu, Philip S},
  booktitle={Proceedings of the 44th international ACM SIGIR conference on Research and development in information retrieval},
  pages={2141--2145},
  year={2021}
}

@inproceedings{ReLU,
  title={Rectified linear units improve restricted boltzmann machines},
  author={Nair, Vinod and Hinton, Geoffrey E},
  booktitle={Proceedings of the 27th international conference on machine learning (ICML-10)},
  pages={807--814},
  year={2010}
}

@article{grl_survey,
  title={Graph representation learning: a survey},
  author={Chen, Fenxiao and Wang, Yun-Cheng and Wang, Bin and Kuo, C-C Jay},
  journal={APSIPA Transactions on Signal and Information Processing},
  volume={9},
  pages={e15},
  year={2020},
  publisher={Cambridge University Press}
}

@article{ENZYMES,
  title={Protein function prediction via graph kernels},
  author={Borgwardt, Karsten M and Ong, Cheng Soon and Sch{\"o}nauer, Stefan and Vishwanathan, SVN and Smola, Alex J and Kriegel, Hans-Peter},
  journal={Bioinformatics},
  volume={21},
  number={suppl\_1},
  pages={i47--i56},
  year={2005},
  publisher={Oxford University Press}
}

@article{SCORE,
  title={Towards Graph Foundation Models: The Perspective of Zero-shot Reasoning on Knowledge Graphs},
  author={Wang, Kai and Luo, Siqiang},
  journal={arXiv preprint arXiv:2410.12609},
  year={2024}
}

@inproceedings{All-in-One,
  author       = {Xiangguo Sun and
                  Hong Cheng and
                  Jia Li and
                  Bo Liu and
                  Jihong Guan},
  title        = {All in One: Multi-task Prompting for Graph Neural Networks (Extended
                  Abstract)},
  booktitle    = {Proceedings of the Thirty-Third International Joint Conference on
                  Artificial Intelligence, {IJCAI} 2024, Jeju, South Korea, August 3-9,
                  2024},
  pages        = {8460--8465},
  publisher    = {ijcai.org},
  year         = {2024},
  url          = {https://www.ijcai.org/proceedings/2024/942},
  bibsource    = {dblp computer science bibliography, https://dblp.org}
}

@article{graph-pretrain-survey,
  title={A survey of pretraining on graphs: Taxonomy, methods, and applications},
  author={Xia, Jun and Zhu, Yanqiao and Du, Yuanqi and Li, Stan Z},
  journal={arXiv preprint arXiv:2202.07893},
  year={2022}
}

@inproceedings{MDP-GNN,
  title={Unified Graph Neural Networks Pre-training for Multi-domain Graphs},
  author={Lin, Mingkai and Hong, Xiaobin and Li, Wenzhong and Lu, Sanglu},
  booktitle={Proceedings of the AAAI Conference on Artificial Intelligence},
  volume={39},
  number={11},
  pages={12165--12173},
  year={2025}
}

@inproceedings{MINE,
  title={On variational bounds of mutual information},
  author={Poole, Ben and Ozair, Sherjil and Van Den Oord, Aaron and Alemi, Alex and Tucker, George},
  booktitle={International conference on machine learning},
  pages={5171--5180},
  year={2019},
  organization={PMLR}
}

@inproceedings{E2Neg,
  title={Does GCL Need a Large Number of Negative Samples? Enhancing Graph Contrastive Learning with Effective and Efficient Negative Sampling},
  author={Huang, Yongqi and Zhao, Jitao and He, Dongxiao and Jin, Di and Huang, Yuxiao and Wang, Zhen},
  booktitle={Proceedings of the AAAI Conference on Artificial Intelligence},
  volume={39},
  number={16},
  pages={17511--17518},
  year={2025}
}

@inproceedings{Str-GCL,
  title={Str-GCL: Structural Commonsense Driven Graph Contrastive Learning},
  author={He, Dongxiao and Huang, Yongqi and Zhao, Jitao and Wang, Xiaobao and Wang, Zhen},
  booktitle={Proceedings of the ACM on Web Conference 2025},
  pages={1129--1141},
  year={2025}
}

@article{UniPrompt,
  title={One Prompt Fits All: Universal Graph Adaptation for Pretrained Models},
  author={Huang, Yongqi and Zhao, Jitao and He, Dongxiao and Wang, Xiaobao and Li, Yawen and Huang, Yuxiao and Jin, Di and Feng, Zhiyong},
  journal={arXiv preprint arXiv:2509.22416},
  year={2025}
}

@article{SGRL,
  title={Exploitation of a latent mechanism in graph contrastive learning: Representation scattering},
  author={He, Dongxiao and Shan, Lianze and Zhao, Jitao and Zhang, Hengrui and Wang, Zhen and Zhang, Weixiong},
  journal={Advances in Neural Information Processing Systems},
  volume={37},
  pages={115351--115376},
  year={2024}
}

\appendix

\section{Experimental Settings}
\label{app_exp_settings}

\begin{table*}[ht]
\centering
\caption{Dataset statistics of node and graph classification.}
\resizebox{0.80\textwidth}{!}{
\begin{tabular}{lcccccc}
\toprule
\multicolumn{7}{c}{\textbf{Node-level}}\\
\midrule
\textbf{Dtasets}   & \textbf{Graphs} & \textbf{Nodes} & \textbf{Edges} & \textbf{Features} & \textbf{Classes} & \textbf{Domain} \\
\midrule
Cora        &1    & 2,708          & 4,732          & 1,433             & 7 & Citation networks              \\
CiteSeer    &1    & 3,327          & 5,429          & 3,703             & 6    & Citation networks            \\
PubMed     &1     & 19,717         & 44,338         & 500               & 3          & Citation networks      \\
Photo      &1     & 7,650          & 119,081        & 745               & 8        & Co-purchase networks        \\
Computers   &1    & 13,752         & 245,861        & 767               & 10      & Co-purchase networks        \\
CS      &1        & 18,333         & 81,894         & 6,805             & 15      & Co-author networks        \\
Physics   &1      & 34,493         & 247,962        & 8,415             & 5      & Co-author networks         \\
\midrule
\multicolumn{7}{c}{\textbf{Graph-level}}\\
\midrule
\textbf{Dtasets}   & \textbf{Graphs} & \textbf{Nodes} & \textbf{Edges} & \textbf{Features} & \textbf{Classes} &\textbf{Domain} \\
\midrule
IMDB-BINARY      & 1,000          & 19.8              & 96.53             & 0     &2 & Social networks \\
COLLAB           & 5,000          & 74.5              & 2457.8            & 0   & 3 & Social networks \\
PROTEINS         & 1,113          & 39.1              & 72.8              & 3  & 2 & Biological networks\\
ENZYMES          & 600            & 32.6              & 62.1              & 3   & 6  & Biological networks\\
\bottomrule
\end{tabular}
}
\label{tab:datasets}
\end{table*}

\subsection{Implementation Details}
\label{app_implementation_detail}
In this subsection, we elaborate on the details of our experimental implementation.
In different scenarios, we follow different prior works, as each of them only covers a subset of our scenarios.
For some methods whose reproduced performance was lower than reported, we instead report their best-known accuracy for fairness.
All experiments were conducted on NVIDIA GeForce GTX 3090 GPU (24GB memory).
\textbf{For the cross-domain node classification task}, we follow \cite{BGRL}. 
We first freeze the parameters of the pre-trained LEDA model and obtain node representations for the test data.
Next, we train a downstream classifier using 10\% of the test data, where the classifier is a simple linear model optimized with a logistic regression loss. We then evaluate the model on the remaining 90\% of the data. We run LEDA 20 times and report the mean and standard deviation of the results.
\textbf{For the cross-domain few-shot node classification task}, we follow \cite{MDGPT}.
We freeze the parameters of the pre-trained LEDA model and obtain node representations for the test data.
For each class, we randomly select and label $k$ samples, and compute their average as the class prototype vectors.
Final predictions are made by measuring the similarity between each test node and the class prototypes.
Moreover, it is worth noting that, compared to \cite{MDGPT}, our training domain includes the Co.CS dataset, which introduces greater domain diversity and poses a more challenging setting for evaluating the ability to learn generalizable knowledge across domains.
To ensure statistical robustness, we run 500 times and report the mean and standard deviation.
\textbf{For the cross-domain zero-shot graph classification task}, we follow \cite{SCORE}. 
We keep the pre-trained LEDA model frozen and derive graph-level representations via a pooling function. 
Predictions are made based on their similarity to class prototype vectors. 

Moreover, for all the scenarios above, we adopt a single-layer GCN as the encoder for the LEDA model and use AdamW \cite{adamw} as the optimizer. To ensure reproducibility, we fix the random seed to $66666$. 
Besides, in implementation, due to potentially differences in topologies across datasets, we apply different numbers of additional propagation steps after obtaining initial representations. 
It is worth noting that we use a shared set of hyper-parameters across all datasets in each scenario.

\subsection{Supplementary Experimental Results}
\label{app_supplementary_results}

\textbf{Detailed values corresponding to the 3-shot and 5-shot plots in Figure \ref{fig:few_shot}.}
We present the detailed accuracy of LEDA under the 3-shot and 5-shot settings in Tables \ref{tab:3-shot-1} and \ref{tab:5-shot-1}.
\begin{table*}[h]
\centering
\caption{Performance on cross-domain 3-shot node classification.  Each column represents a test domain, while others are train domains. The best results are bolded and the second-best are underlined.}
\resizebox{1.0\textwidth}{!}{
\begin{tabular}{lccccccc}
\hline
\textbf{Method} \textbackslash \textbf{Test data}  & \textbf{Cora} & \textbf{CiteSeer} & \textbf{PubMed} & \textbf{Photo} & \textbf{Computers} & \textbf{CS} \\
\hline
DGI  & 33.00 $\pm$ 4.79 & 34.39 $\pm$ 4.73	&53.12 $\pm$ 6.73 &56.88 $\pm$  5.95	&37.55 $\pm$  5.49 &76.44 $\pm$ 2.54\\
BGRL  & 43.84 $\pm$ 4.99	&46.71 $\pm$ 5.41	& 54.86 $\pm$ 6.30	& 58.58 $\pm$ 6.43	& 45.65 $\pm$ 7.96  &79.69 $\pm$ 2.73\\
GraphMAE  & 47.44 $\pm$ 6.18 &49.02 $\pm$ 5.77 & 60.10 $\pm$ 6.43 &67.96 $\pm$ 7.99 &55.41 $\pm$ 10.07 &81.95 $\pm$ 2.58 \\
\hline
FUG  & \underline{57.92 $\pm$ 6.80}  & \underline{49.85 $\pm$ 5.45} & \underline{61.96 $\pm$ 5.77} & \textbf{68.86 $\pm$ 7.39} & \underline{56.20 $\pm$ 9.73} & \underline{85.67 $\pm$ 1.85}\\
\textbf{LEDA} & \textbf{58.91 $\pm$ 7.15}	&\textbf{57.42 $\pm$ 6.45}	&\textbf{67.41 $\pm$ 7.91}	& \underline{68.39 $\pm$ 8.43}	& \textbf{56.34 $\pm$ 12.05} & \textbf{86.75 $\pm$ 1.98} \\
\hline
\end{tabular}
}
\label{tab:3-shot-1}
\end{table*}

\begin{table*}[h]
\centering
\caption{Performance on cross-domain 5-shot node classification.  Each column represents a test domain, while others are train domains. The best results are bolded and the second-best are underlined.}
\resizebox{1.0\textwidth}{!}{
\begin{tabular}{lccccccc}
\hline
\textbf{Method} \textbackslash \textbf{Test data}  & \textbf{Cora} & \textbf{CiteSeer} & \textbf{PubMed} & \textbf{Photo} & \textbf{Computers} & \textbf{CS} \\
\hline
DGI  & 33.59 $\pm$ 3.97	& 45.64 $\pm$ 4.11	& 55.00 $\pm$ 5.15 & 59.04 $\pm$ 5.42	& 38.79 $\pm$ 5.42 & 79.63 $\pm$ 1.71\\
BGRL  & 45.84 $\pm$ 4.27	&51.64 $\pm$ 3.85	& 58.04 $\pm$ 5.77 & 58.08 $\pm$ 5.11 	& 48.71 $\pm$ 7.17 & 83.99 $\pm$ 2.24 \\ 
GraphMAE & 52.73 $\pm$ 4.66 &52.22 $\pm$ 5.58 &63.62 $\pm$ 5.28 & 68.48 $\pm$ 8.26 & 57.54$\pm$10.21 & 84.10 $\pm$ 1.71  \\
\hline
FUG  & \underline{64.95 $\pm$ 4.70} & \underline{55.64 $\pm$4.16} & \underline{64.51 $\pm$ 4.84} & \underline{70.27 $\pm$ 6.29} & \underline{58.32 $\pm$ 8.64} & \underline{88.04 $\pm$ 1.23}\\
\textbf{LEDA} & \textbf{66.07 $\pm$  5.72}	&\textbf{62.54 $\pm$ 3.26}	&\textbf{70.62 $\pm$  6.07}   & \textbf{71.08 $\pm$ 7.51}	& \textbf{58.79 $\pm$  11.8} & \textbf{88.47 $\pm$  1.41}	\\
\hline
\end{tabular}
}
\label{tab:5-shot-1}
\end{table*}

\begin{table*}[h]
\centering
\caption{Performance on cross-domain few-shot node classification.  The best results are bolded and the second-best are underlined.}
\resizebox{1.0\textwidth}{!}{
\begin{tabular}{lccccccc}
\toprule
\multicolumn{7}{c}{\textbf{Cross-domain 1-shot node classification}}\\
\midrule
\textbf{Method} \textbackslash \textbf{Test data}  & \textbf{Cora} & \textbf{CiteSeer} & \textbf{PubMed} & \textbf{Photo} & \textbf{Computers} & \textbf{CS} \\
\toprule
DGI &27.93$\pm$4.83	&26.33$\pm$4.08	&40.62$\pm$5.39		&34.50$\pm$6.82	&24.12$\pm$5.00 &33.88$\pm$6.14 \\
BGRL  &35.97$\pm$5.43	&26.03$\pm$3.65	&42.07$\pm$4.39		&38.49$\pm$6.07	&33.49$\pm$7.88 &50.87$\pm$4.72 \\ 
GraphMAE & 37.67$\pm$7.79	&36.43$\pm$8.45	&48.89$\pm$9.20		&55.80$\pm$9.48	&43.22$\pm$11.55 &66.45$\pm$7.58 \\
\midrule
FUG  & \underline{40.88 $\pm$ 8.65}	& \underline{37.31 $\pm$ 7.50}	& \underline{50.38$\pm$9.60}	& \underline{60.46$\pm$10.38}	& \underline{45.72$\pm$11.72} & \underline{72.57$\pm$8.60}	 \\

\textbf{LEDA} & \textbf{50.51$\pm$10.94}	&\textbf{44.83$\pm$10.05}	&\textbf{54.18$\pm$11.03} &\textbf{63.95 $\pm$ 9.64}	&\textbf{50.01$\pm$11.93} &\textbf{76.25$\pm$6.94} \\
\bottomrule
\multicolumn{7}{c}{\textbf{Cross-domain 3-shot node classification}}\\
\toprule
DGI &35.36$\pm$4.67	&30.90$\pm$3.64	&42.79$\pm$4.72		&48.56$\pm$6.02	&27.80$\pm$4.70  &66.67$\pm$3.02\\
BGRL & 44.00$\pm$4.50	&33.28$\pm$3.75	&44.65$\pm$4.32		&44.84$\pm$6.64	&37.24$\pm$8.99  &64.92$\pm$3.21\\ 
GraphMAE & 50.33$\pm$5.96	&45.58$\pm$6.68	&58.77$\pm$6.55		&62.33$\pm$7.66	&55.76$\pm$8.35 &84.43$\pm$1.83\\
\midrule
FUG  & \underline{51.62$\pm$7.15}	& \underline{48.40$\pm$5.65}	& \underline{60.52$\pm$6.52}	& \underline{65.75$\pm$7.33}	& \underline{56.61$\pm$9.80} & \underline{85.97$\pm$1.95}	 \\
\textbf{LEDA} & \textbf{60.05$\pm$7.90}	&\textbf{57.76$\pm$6.04}	&\textbf{64.47$\pm$7.76}	 &\textbf{66.77$\pm$8.48}	&\textbf{58.36$\pm$10.96}  &\textbf{86.43$\pm$2.12} \\
\bottomrule
\multicolumn{7}{c}{\textbf{Cross-domain 5-shot node classification}}\\

\toprule
DGI &41.08$\pm$4.27	&35.93$\pm$3.64	&46.39$\pm$4.13	&38.49$\pm$4.83	&30.01$\pm$5.53 &69.25$\pm$3.22\\
BGRL &50.08$\pm$3.86	&37.18$\pm$3.43	&48.78$\pm$4.05 &46.68$\pm$6.10	&41.63$\pm$8.62 &77.95$\pm$2.26 \\ 
GraphMAE &56.67$\pm$4.68	& 52.86$\pm$3.50	& 60.87$\pm$5.44 &60.12$\pm$8.66 &52.77$\pm$9.78 &83.24$\pm$2.11 \\
\midrule
FUG  & \underline{61.98±5.43}	& \underline{54.19±3.98}	& \underline{63.10±5.34}	& \textbf{70.34±6.20}	& \textbf{62.11±8.49} & \underline{87.95±1.21}	 \\
\textbf{LEDA} & \textbf{63.29$\pm$7.31}	&\textbf{62.67$\pm$3.63}	&\textbf{68.20$\pm$5.56}	&\underline{66.43$\pm$7.90}	&\underline{60.60$\pm$10.62} &\textbf{88.15$\pm$1.46} \\
\bottomrule
\end{tabular}
}
\label{tab:few-shot-2}
\end{table*}

\textbf{Supplementary cross-domain few-shot node classification results.} 
Although the above few-shot cross-domain experiments demonstrate the strong performance of LEDA, the training and testing sets still involve datasets from similar domains, such as the citation networks (Cora, CiteSeer, PubMed) and the co-purchase networks (Photo, Computers).
To better simulate real cross-domain scenarios, we construct a more challenging setting: using citation networks as the training domain and evaluating on co-purchase and co-author networks, and vice versa. 
We conduct experiments under the 1-shot, 3-shot, and 5-shot settings, and report the corresponding results on Table \ref{tab:few-shot-2}.

\begin{table*}
\centering
\caption{Supplementary ablation studies. The best results are bolded.}
\resizebox{1.0\textwidth}{!}{
\begin{tabular}{llccc|ccc}
\toprule
& & \multicolumn{3}{c|}{\textbf{3-shot}} & \multicolumn{3}{c}{\textbf{5-shot}} \\
\cmidrule(r){3-5} \cmidrule(l){6-8}
& & \textbf{Cora} & \textbf{CiteSeer} & \textbf{PubMed} & \textbf{Cora} & \textbf{CiteSeer} & \textbf{PubMed} \\
\midrule
\multirow{4}{*}{\textbf{Cross-domain}} 
& w/o DPU  &44.07$\pm$6.39 &50.91$\pm$5.29 & 60.79$\pm$7.29 &50.56$\pm$5.27 &56.56$\pm$3.72 &65.88$\pm$5.48\\
& w/o LDA  &45.84$\pm$6.01 &52.46$\pm$5.24 &62.25$\pm$6.41 &51.07$\pm$4.48 &57.32$\pm$3.67&65.70$\pm$4.85\\
& DPU+CL  &44.98$\pm$6.01 &52.15$\pm$5.27 &61.52$\pm$6.43 &50.01$\pm$4.58 &57.02$\pm$3.71&65.08$\pm$4.92\\
& \textbf{LEDA} & \textbf{58.91$\pm$7.15}	&\textbf{57.42$\pm$6.45}	&\textbf{67.41$\pm$7.91} & \textbf{66.07$\pm$5.72}	&\textbf{62.54$\pm$3.26}	&\textbf{70.62$\pm$6.07} \\
\bottomrule

\end{tabular}
}
\end{table*}

\section{Proof of Proposition \ref{prop:mi_contrastive}}
\begin{proof}
By definition, the mutual information is:
\begin{equation}
    I(\mathcal{D}_{\mathcal{G}^i}; \mathcal{D}_{\mathcal{G}^j}) = \mathbb{E}_{x, x'} \left[ \log \frac{p(x, x')}{p(x)p(x')} \right].
\end{equation}
Substituting the assumed forms $p(x, x') = e^{s_{ij}} / Z$ and $p(x)p(x') = e^{\xi(x, x')}$ yields:
\begin{equation}
    I(\mathcal{D}_{\mathcal{G}^i}; \mathcal{D}_{\mathcal{G}^j}) = \mathbb{E}_{x, x'} \left[ \log \left( \frac{e^{s_{ij}} / Z}{e^{\xi(x, x')}} \right) \right].
\end{equation}
Simplifying the logarithm gives:
\begin{equation}
    I(\mathcal{D}_{\mathcal{G}^i}; \mathcal{D}_{\mathcal{G}^j}) = \mathbb{E}_{x, x'} \left[ s_{ij} - \log Z - \xi(x, x') \right].
\end{equation}
Applying the linearity of expectation separates the terms:
\begin{equation}
    I(\mathcal{D}_{\mathcal{G}^i}; \mathcal{D}_{\mathcal{G}^j}) = \mathbb{E}[s_{ij}] - \log Z - \mathbb{E}[\xi(x, x')].
\end{equation}
Denoting $\Delta = \mathbb{E}[\xi(x, x')]$, we obtain:
\begin{equation}
    I(\mathcal{D}_{\mathcal{G}^i}; \mathcal{D}_{\mathcal{G}^j}) = \mathbb{E}[s_{ij}] - \log Z - \Delta.
\end{equation}
\end{proof}

\section{Supplementary of Discussion}
\label{app_supplement_dis}
\textbf{Discussion of the potential semantic conflicts across domains.}
The potential semantic conflicts across domains has been acknowledged in previous work \cite{MDGPT}.
A key challenge in this context lies in the fact that features which are semantically meaningful in one domain may carry entirely different implications in another.
For instance, a densely connected subgraph in a citation network like Cora may reflect topic coherence among papers, while in a co-purchase network such as Amazon Computers, a similar structure might arise from shared buying patterns that do not necessarily reflect category-level similarity.
Such discrepancies can mislead the encoder during pre-training. This results in representations that conflate unrelated semantics, ultimately impairing their transferability and weakening downstream performance.

\textbf{Discussion on the Reasonableness of the Distributional Assumption in Contrastive Learning.}
The assumption that the joint distribution $p(x, x')$ can be modeled via a similarity function $s_{ij} = f(x, x')$, and the marginal distribution $p(x)p(x')$ is approximated by $C \cdot e^{\xi(x,x')}$, serves as a theoretical foundation for analyzing contrastive learning objectives. This formulation aligns with the principle behind InfoNCE and related objectives, which aim to maximize mutual information by contrasting positive and negative pairs. Specifically, interpreting the similarity function as a proxy for the log joint probability allows the contrastive loss to be viewed as an estimator of mutual information between different views of the same instance.
The introduction of the bias term $\xi(x,x')$ offers additional flexibility, acknowledging the potential discrepancy between the assumed and true marginal distributions. While this term is not explicitly computed in practice, it enables a more precise theoretical analysis by isolating sources of deviation from ideal assumptions. Furthermore, similar assumptions have been employed in previous works to derive theoretical guarantees and bounds for representation quality and generalization.
Although this modeling does not reflect an exact probabilistic characterization of the data, it provides a mathematically tractable abstraction that facilitates a deeper understanding of the learning dynamics. Therefore, this assumption is reasonable and widely accepted in the theoretical study of contrastive representation learning \cite{DGI,GRACE,MINE}.

\end{document}